\documentclass[english,a4paper,12pt]{article}
\usepackage{graphicx}
\usepackage[english]{babel}
\usepackage{amsmath} 
\usepackage{amsxtra, amsfonts, amssymb, amstext}
\usepackage{amsthm}
\usepackage{booktabs}
\usepackage{nicefrac}
\usepackage{xspace}
\usepackage{url}
\usepackage{graphics,color}
\usepackage[algo2e,ruled,vlined,linesnumbered]{algorithm2e}
\newcommand{\assign}{\leftarrow}
\usepackage{wrapfig}
\usepackage{enumitem}
\usepackage{multirow}

\allowdisplaybreaks[4]
\newtheorem{theorem}{Theorem}
\newtheorem{example}[theorem]{Example}

\newtheorem{lemma}[theorem]{Lemma}

\newtheorem{corollary}[theorem]{Corollary}

\newcommand{\ignore}[1]{}

\newcommand{\N}{\mathbb{N}}

\newcommand{\R}{\mathbb{R}}
\newcommand{\Z}{\mathbb{Z}}
\renewcommand{\epsilon}{\varepsilon}
\newcommand{\eps}{\varepsilon}
\DeclareMathOperator{\mut}{mut}
\newcommand{\pmut}{p_{\mut}}

\newcommand{\X}{\mathcal{X}}

\DeclareMathOperator{\Bin}{Bin}

\newcommand{\onemax}{\textsc{OneMax}\xspace}

\newcommand{\leadingones}{\textsc{LeadingOnes}\xspace}

\newcommand{\mclea}{\mbox{$(\mu , \lambda)$~EA}\xspace}

\DeclareMathOperator{\Var}{Var}

\newcommand{\set}[2]{\{ #1 \mid #2\}}

\begin{document}

\title{Multiplicative Up-Drift\thanks{Extended and improved version of a paper that appeared in the proceedings of GECCO 2019~\cite{DoerrK19}. }
}
\author{Benjamin Doerr\setcounter{footnote}{6}\thanks{Laboratoire d'Informatique (LIX), CNRS, \'Ecole Polytechnique, Institut Polytechnique de Paris, Palaiseau, France} \and Timo K{\"o}tzing\thanks{Hasso Plattner Institute, Potsdam, Germany}}

%
%
%
%

\maketitle

\begin{abstract}
Drift analysis aims at translating the expected progress of an evolutionary algorithm (or more generally, a random process) into a probabilistic guarantee on its run time (hitting time). So far, drift arguments have been successfully employed in the rigorous analysis of evolutionary algorithms, however, only for the situation that the progress is constant or becomes weaker when approaching the target.  

Motivated by questions like how fast fit individuals take over a population, we analyze random processes exhibiting a $(1+\delta)$-multiplicative growth in expectation. We prove a drift theorem translating this expected progress into a hitting time. This drift theorem gives a simple and insightful proof of the level-based theorem first proposed by Lehre (2011). Our version of this theorem has, for the first time, the best-possible near-linear dependence on $1/\delta$ (the previous results had an at least near-quadratic dependence), and it only requires a population size near-linear in~$\delta$ (this was super-quadratic in previous results). These improvements immediately lead to stronger run time guarantees for a number of applications.

We also discuss the case of large $\delta$ and show stronger results for this setting.
\end{abstract}


\renewcommand{\descriptionlabel}[1]{\hspace{\labelsep}\textup{\textbf{#1}}}

{\sloppy

\section{Introduction}%
\label{sec:introduction} 

In a typical situation in evolutionary search, an algorithm first makes good progress while far away from the target, since a lot can still be improved. As the search focuses more and more on the fine details, progress slows and finding improving moves becomes rarer. Thus, the expected progress is typically an increasing function of the distance from the optimum. However, there are also many processes where this situation is reversed. For example, for heuristics involving a population, once a superior individual is found, this improvement needs to be spread over the population. This process gains speed when more individuals exist with the improvement. 

Turning expected progress into an expected first hitting time is the purpose of drift theorems (see the recent survey~\cite{Lengler20bookchapter} for a thorough introduction to drift analysis). For example, the additive drift theorem~\cite{He-Yao:j:01,He-Yao:j:04} 
 requires a uniform lower bound $\delta$ on the expected progress (the \emph{drift}) and gives an expected first hitting time of at most $n/\delta$, where $n$ is the initial distance from the optimum. This theorem can also be applied when the drift is changing during the process, but since a uniform $\delta$ is used in the argument, the additive drift theorem cannot be used to exploit a stronger drift later in the process. 

A first step towards profiting from a changing drift behavior was the multiplicative drift theorem~\cite{DoerrJW12algo,DoerrG13algo}. It assumes that the drift is at least $\delta x$ when the distance from the optimum is $x$, for some factor $\delta<1$. The first hitting time can then be bounded by $O(\log (n)/\delta)$, where $n$ is again the initial distance from the optimum. Apparently, this gives a much better bound than what could be shown via the additive drift in this setting. Multiplicative drift can be found in many optimization processes, making the multiplicative drift theorem one of the most useful drift theorems.

To cope with a broader variety of changing drift patterns, the variable drift theorem \cite{MitavskiyRC09,Joh:th:10} has been developed. However, while there are several variants of this drift theorem, most of them require that the strength of the drift is a monotone \emph{increasing} function in the distance from the optimum (the farther away from the optimum, the easier it is to make progress).

In this paper we are concerned with the reverse setting where drift is a \emph{decreasing} function of the distance from the optimum. This has been considered only for few variable drift theorems, and all of them essentially require a step-size bounded processes. The most recent formulation of this can be found in \cite{OlivetoW15}. We want to consider processes which are not step-size bounded, so this drift theorem cannot be usefully applied.

While many drift theorems are phrased such that the aim is to reach the point zero, for our setting it is more natural to consider the case of reaching some target value~$n$ starting at a value of $1$, and to suppose that the drift is $\delta x$ \emph{going up} (for the multiplicative drift theorem, we had a drift of $\delta x$ \emph{going down}). Thus, we call our resulting drift theorem  the \emph{multiplicative up-drift theorem}. 

Making things more formal, consider a random process $(X_{t})_{t \in \N}$ over positive reals starting at $X_{0} = 1$ and with target $n > 1$. We speak of \emph{multiplicative up-drift} if there is a $\delta > 0$ such that, for all $t \ge 0$, we have the drift condition
\begin{description}
	\item[(D)] $E[X_{t+1} - X_{t} \mid X_{t}] \geq \delta X_{t}$. 
\end{description}
Note that this is equivalent to 
\begin{description}
	\item[(D')] $E[X_{t+1} \mid X_{t}] \geq (1+\delta) X_{t}$.
\end{description}
One trivial case of any drift process is the deterministic process with the desired gain per iteration. We quickly regard this case now as it gives the right impression of what should be a natural expected first hitting time for a well-behaved process exhibiting multiplicative up-drift.

\begin{example}\label{ex:deterministicProcess}
Let $\delta > 0$. Suppose $X_{0} = 1$ and, for all $t$, $X_{t+1} = (1+\delta) X_{t}$ with probability~$1$. Then this process satisfies the drift condition \textbf{(D)} with equality. Clearly, the time to reach a value of at least $n$ is $\lceil \log_{1+\delta}(n) \rceil$. For small $\delta$, this is approximately $\log(n) / \delta$, for large $\delta$, it is approximately $\log(n) / \log(\delta)$. We note here already that we will be mostly concerned with the case where $\delta$ is small. This case is the harder one since the progress is weaker, and thus there is a greater need for stronger analysis tools in this case. 
\end{example}

Unfortunately, not all processes with multiplicative up-drift have a hitting time of $O(\log(n) / \delta)$, as the following example shows. 

\begin{example}\label{ex:highVarianceProcess}
Let $\delta > 0$. Suppose $X_{0} = 1$ and, for all $t$, $X_{t+1} = n$ with probability $\delta/(n-1)$ (which we term a \emph{success}) and $X_{t+1} = 1$ otherwise. Again, the drift condition \textbf{(D)} is satisfied with equality (while the target $n$ is not reached). The time for the process to hit the target $n$ is thus geometrically distributed with probability $\delta/(n-1)$, giving an expected time of $(n-1)/\delta = \Theta(n/\delta)$ iterations, significantly more than the $O(\log(n)/\delta)$ seen in the deterministic process.
\end{example}

\ignore{
\begin{example}\label{ex:highVarianceProcess}
Let $\delta > 0$. Suppose $X_{0} = 1$ and, for all $t$, $X_{t+1} = 2(1+\delta)X_{t} - 1$ with probability $0.5$ (which we term a \emph{success}) and $X_{t+1} = 1$ otherwise. Again, the drift condition \textbf{(D)}  is satisfied with equality. A straightforward induction shows that, after $k$ successes, the process has a value of $1 + 2\delta \sum_{i=0}^{k-1}(2+2\delta)^i$. Thus, we require a sequence of about $\log_{2+2\delta} (n/\delta)$ consecutive successes to reach a value of $n$ which, for values of $\delta = o(1/\log n)$, has a probability of about $2^{-\log_{2+2\delta} (n/\delta)} \approx \delta/n$. Therefore we expect to need $\Omega(n/\delta)$ iterations, significantly more than the $O(\log(n)/\delta)$ seen in the deterministic process.
\end{example}
}
Note that for this process the additive drift theorem immediately gives the upper bound of $O(n/\delta)$ since we always have a drift of at least $\delta$ towards the target. Hence Example~\ref{ex:highVarianceProcess} describes a process where the stronger assumption of multiplicative up-drift does not lead to a better hitting time.

Our first main result (Theorem~\ref{thm:reverseDriftBinomial}) shows that the targeted bound of $O(\log (n) / \delta)$, which as we saw is optimal when we want to cover the deterministic process given in Example~\ref{ex:deterministicProcess}, can be obtained when strengthening condition~\textbf{(D)} by assuming (i)~that, given $X_{t}$, the next state $X_{t+1}$ is at least (in the stochastic domination sense) binomially distributed with expectation $(1+\delta) X_{t}$, and (ii)~that the process never reaches state~$0$. The first condition is very natural. When generating offspring independently, the number of offspring satisfying a particular desired property is binomially distributed. The second condition is a technical necessity. From the up-drift condition alone, we cannot infer any progress from state~$0$. Consequently, $0$ could well be an absorbing state, resulting in an infinite hitting time if this state can be reached with positive probability. 

In quite some applications, however, we cannot rule out that the random process reaches state~$0$. For example, when regarding the subpopulation of individuals having some desired property, then in an algorithm using comma selection, this might die out completely in one iteration (though often with small probability only). To cover also such processes, in our second drift theorem (Theorem~\ref{thm:reverseDriftWithZeroBinomial}) we extend our Theorem~\ref{thm:reverseDriftBinomial} to include that state~$0$ is reached with at most the probability that can be deduced from the up-drift and the binomial distribution conditions. To avoid that state~$0$ is absorbing, we add an additional condition governing how this state~$0$ is left again (see Theorem~\ref{thm:reverseDriftWithZeroBinomial} for the precise statement).
%

As mentioned before, a main application for multiplicatively increasing drift towards the optimum is the analysis of how fit~individuals spread in a population. This particular setting was previously analyzed as the \emph{level-based theorem}~\cite{Leh:c:11:LevelBased,DangL16,CorusDEL18plus}, modeled after the method of fitness-based partitions~\cite{Wegener01}. Essentially, the search space is partitioned into an ordered sequence of \emph{levels}. The ongoing search process increases the probability that a newly-created individual is \emph{at least} on a given level and, once this probability is sufficiently high, that there is a good chance that the individual is on an even higher level. We restate the details of this theorem in the version from~\cite{CorusDEL18plus} in Theorem~\ref{thm:OldLevelBasedTheorem} below. The level-based theorem was originally intended for the analysis of non-elitist population-based algorithms \cite{DangL16}, but has since also been applied to EDAs, namely to the UMDA in \cite{DangLN19} and, with some additional arguments, 
to PBIL in \cite{LehreN18}.

We use our second multiplicative up-drift theorem (Theorem~\ref{thm:reverseDriftWithZeroBinomial}) to prove a new version of the level-based theorem (Theorem~\ref{thm:NewLevelBasedTheorem}). This new theorem allows to derive better asymptotic bounds under mostly weaker conditions: The dependence of the run time on $1/\delta$ is reduced from near-quadratic to near-linear\footnote{We use the prefix ``near'' to suppress that in some cases, an additional factor of order $\log(1/\delta)$ is present.} and the minimum population size $\lambda$ required for the result to hold is reduced from super-quadratic in $1/\delta$ to near-linear in $1/\delta$. Since the run time often is linear in $\lambda$, this can give a further run time improvement. Our upper bounds almost match the lower-bound example given in \cite{CorusDEL18plus} and, in particular, match the asymptotic dependence on $\delta$ displayed by this example. 

Our version of the level-based theorem can be applied in all settings where the previous-best level-based theorems were used. It leads to better results when $\delta$ is small. In Section~\ref{sec:applications}, we analyze two such situations from  previous analyses of non-elitist evolutionary algorithms on standard test functions. The first test function is called \onemax and maps a given bit string to the number of $1$s in that bit string, thus simulating a unimodular optimization problem solvable by simple hill climbing. The second test function is called \leadingones and maps a bit string to the number of $1$s appearing in the bit string before the first $0$ (if any); this simulates an optimization problem requiring sequential optimization of different sub parts. Our results are as follows. (i)~We prove that the $(\lambda,\lambda)$ EA with fitness-proportionate selection and suitable parameters can optimize the \onemax and \leadingones functions in expected time $O(n^3 \log^2 n)$ and $O(n^4)$ respectively, improving over the previous-best published bound of $O(n^8 \log n)$. (ii)~We prove that the $(\lambda,\lambda)$ EA with $2$-tournament selection and suitable parameters in the restricted setting that only a constant fraction of the bits of the search points are evaluated finds the optimum of \onemax in $O(n^{2.5} \log^2 n)$ iterations. The previous-best published bound here is $O(n^{4.5} \log n)$.

We also use our methods to obtain a level-based theorem for the case that $\delta$ is large (Theorem~\ref{thm:NewLevelBasedTheoremDeltaLarge}). This case was not covered by the previous-best level-based theorems and our theorem now allows to exploit larger values of $\delta$ to obtain asymptotically stronger run time guarantees. As an example we show (in Section~\ref{sec:usingLargeDelta}) that the \mclea with $\mu = n$ and $\lambda=n^{1.5}$ on the \leadingones benchmark function using ranking selection and standard bit mutation has an optimization time of $O(n^{2.5})$. This is asymptotically better than the previously known bound of $O(n^{2.5} \log(n))$ and also shows more explicitly how optimization proceeds.

Beyond these particular results, our modular proof (first analyzing the multiplicative up-drift excluding $0$, then including $0$, then applying it in the context of the level-based theorem) shows the level-based theorem in a way that is more accessible than the previous versions and that gives more insight into population-based optimization processes.

In particular, our proof suggests that the behavior of the process under the named conditions is as follows.
\begin{itemize}
	\item Once a critical mass in a level is reached, this level is never again abandoned. Thus, we can focus in our analysis on having a critical mass of individuals in one level and analyze the time it takes to gain a critical mass in the next level.
	\item Reaching a critical mass in the next level consists of two steps.
		\begin{enumerate}
			\item[1.] When few elements are in the next level, then these elements go extinct regularly and need to be respawned until this initial population on this level via a mostly unbiased random walk gains a moderate amount of elements.
			\item[2.] With this moderate amount of elements, the bias of the random walk is large enough to make a significant decrease of the population unlikely, but instead the number of elements increases steadily, as can be shown using a concentration bound for submartingales, so that we quickly gain a critical mass in the next level.
		\end{enumerate}
\end{itemize}
We are optimistic that this increased understanding of population-based processes helps in the future design and analysis of such processes.

\ignore{
\subsection{Remaining Ideas}

Maybe we even have a third goal: make it applicable to other settings (PBIL, $\mu+1$ EA, \ldots). Certainly we have the goal of presenting the reverse-multiplicative drift theorem independently of its use in the level-based theorem.

Related lines of research (or application areas) to be mentioned in the journal version:
\begin{itemize}
	\item Rumor spreading;
	\item The Moran process \cite{Mor1958:Moran} (essentially rumor spreading with the possibility of forgetting the rumor);
	\item Reaching a new level is a process not unlike the Galton-Watson process (see the text book \cite{Har:b:02} for an introduction to the Galton-Watson process).
\end{itemize}
}

\section{Multiplicative Up-Drift Theorems}

In this section we prove three multiplicative up-drift theorems. The first is concerned with processes that cannot reach the value $0$ (which could be absorbing if only a multiplicative up-drift assumption is made); the second one extends the first theorem to include also the possibility of going down to $0$ (but taking an additional assumption how state~$0$ is left). The third does the same, but exploits the assumption that, with some positive probability, state~$0$ is left to a state from which, with constant probability, we make strong multiplicative progress in every iteration until the process reaches the target (as opposed to a behavior closer to an unbiased random walk).

Note that our theorems essentially deal with martingales, but still we suppress the mention of conditioning on all previous members of the given process (i.e.\ the natural filtration) to improve readability.

\subsection{Processes on the Positive Integers}

As discussed in the introduction, an expected multiplicative increase as described by~\textbf{(D)} is not enough to ensure the run time we aim at. For this reason, we assume that there is a number $k$ such that, conditional on~$X_{t}$, the next state $X_{t+1}$ is binomially distributed with parameters $k$ and $(1+\delta) X_{t} / k$. Note that this implies~\textbf{(D)}. Since often precise distributions are hard to specify, we only require that $X_{t+1}$ is at least as large as this binomial distribution, that is, we require that $X_{t+1}$ stochastically dominates $\Bin(k, (1+\delta) X_{t} / k)$. See~\cite{Doerr19tcs} for an introduction to stochastic domination and its use in run time analysis. To avoid that the process reaches the possibly absorbing state~$0$, we explicitly forbid this, that is, we require that all $X_{t}$ take values only in the positive integers.

Under these conditions, we analyze the time the process takes to reach or overshoot a given state $n$. For technical reasons, we require that $n$ is not too close to $k$, that is, that there is a constant $\gamma_0 < 1$ such that $n-1 \le \gamma_0 k$. For the trivial reason that the condition $X_{t+1} \succeq \Bin(k, (1+\delta) X_{t} / k)$ does not make sense for $X_{t} > (1+\delta)^{-1} k$, we also require $n-1 \le (1+\delta)^{-1} k$. For all such $n$, we show that an expected number of $O(\log(n)/\delta)$ iterations suffices to reach~$n$ when $\delta \le 1$ and $O(\log(n)/\log(1+\delta))$ iterations suffice for $\delta > 1$. More precisely, we show the following estimate.

\begin{theorem}[First Multiplicative Up-Drift Theorem]
\label{thm:reverseDriftBinomial}\label{thm:first}
Let $(X_{t})_{t \in \N}$ be a stochastic process over the positive integers. Assume that there are ${n,k \in \Z_{\geq 1}}$, $\gamma_0 < 1$, and $\delta > 0$ such that $n -1 \le \min\{\gamma_0 k, (1+\delta)^{-1} k\}$ and for all $t \ge 0$ and all $x \in \{1, \dots, n-1\}$ with $\Pr[X_{t} = x] > 0$ we have the binomial condition
\begin{description}
	\item[(Bin)] $(X_{t+1} \mid X_{t} = x) \succeq \mathrm{Bin}(k,(1+\delta) x/k)$.
\end{description}
Let $T := \min\{t \geq 0 \mid X_{t} \geq n\}$.
\begin{enumerate}
\item If $\delta \le 1$, then with $D_0 = \min\{\lceil 100/\delta \rceil, n\}$ we have
\[
E[T] \le \tfrac{21.6}{1-\gamma_0} D_0 \ln(2 D_0)+ 3.6 \log_2(n) \lceil 3 / \delta \rceil.
\]
If $n > 100/\delta$, then we also have that once the process has reached state of at least $100/\delta$, the probability to ever return to a state of at most $50/\delta$, is at most $0.7218$.
\item If $\delta > 1$, then 
$$
E[T] \le 2.6 \log_{1+\delta}(n) + 81.
$$
In addition, once the process has reached state $32$ or higher, the probability to ever return to a state lower than $32$ is at most $\tfrac{1}{e(e-1)} < 0.22$.
\end{enumerate}
\end{theorem}
For the analysis we will employ Lemma~\ref{lem:nodrift} from Section~\ref{sec:nodrift} essentially for the time spent below $D_0$. Note that this lemma all by itself, in case of $\delta \le 1$ and $n \le D_0$,  gives the stronger bound $E[T] \le \frac{6n \ln(n)}{1-\gamma_0}$.

Since the case $\delta \le 1$ is significantly more complicated, we focus on this case in Sections~\ref{ssec:motiv} to~\ref{ssec:prooffirst} and discuss the case $\delta > 1$ only in Section~\ref{ssec:deltalarge}.

\subsubsection{A Motivating Example}\label{ssec:motiv}

Before proving this result, let us give a simple example of a possible application. Consider the following elitist $(\mu,\lambda)$ EA. It starts with a parent population of $\mu$ individuals chosen uniformly and independently from $\{0,1\}^n$. In each iteration, it generates $\lambda$ offspring, each by independently and uniformly choosing a parent individual and mutating it via standard bit mutation with the usual mutation rate $1/n$. If the offspring population contains at least one individual that is at least as good as the best parent (in terms of fitness), then the new parent population is chosen by selecting $\mu$ best offspring (breaking ties arbitrarily). If all offspring are worse than the best parent, then the new parent population is composed of a best individual from the old parent population and $\mu-1$ best offspring (again, breaking all ties randomly). 

We now use the above theorem to analyze the spread of fit individuals in the parent population. Let us assume that at some time, the parent population contains at least one individual of at least a certain fitness. We shall call such individuals \emph{fit} in the following. Recall that standard bit mutation creates a copy of the parent individual with probability $1/e_n := (1-1/n)^n \approx 1/e$. Hence if the parent population contains $x$ fit individuals, the number of fit individuals in the offspring population is at least (in the domination sense) $\Bin(\lambda, \frac{x}{\mu e_n})$. Due to the elitist selection mechanism, it is also always at least one. Let us assume that $\frac{\lambda}{\mu e_n}$ is greater than one so that the expected number $x \frac{\lambda}{\mu e_n}$ of fit individuals shows a positive drift. Writing  $(1+\delta) := \frac{\lambda}{\mu e_n}$, where $\delta > 0$ by our assumption, and assuming for simplicity $\delta \le 1$ as well, we can apply the first up-drift theorem with $k = \lambda$ and $n = \mu$ and observe that after an expected number of $O(\log(\mu)/\delta)$ iterations, the parent population consists of only fit individuals. 

\subsubsection{Proof Overview}

We now proceed towards proving the first up-drift theorem. As said earlier, we concentrate on the case $\delta \le 1$ in all of the following except Section~\ref{ssec:deltalarge}. We start by outlining the two main difficulties and solutions in a high-level language. 

One of the main difficulties is that the drift towards the target is negligibly weak in the early stages of the process. To demonstrate this, assume that $\delta = o(1)$ and that $X_{t} = o(1/\delta)$. Then the up-drift condition~\textbf{(D)} only ensures a drift of $E[X_{t+1} - X_{t} \mid X_{t}] \ge \delta X_{t} = o(1)$. At the same time, the binomial condition~\textbf{(Bin)} allows a variance $\Var[X_{t+1} \mid X_{t}]$ of order $X_{t}$, or, more specifically, admits deviations of $X_{t+1}$ from its expectations of order $\sqrt{X_t}$ with constant probability. For this reason, in this regime we do not progress because of the drift, but rather because of the random fluctuations of the process. 

It is well-known that random fluctuations are enough to reach a target, with a classical example being the unbiased random walk $(W_{t})$ on the line $[0..n] := \{0, 1, \dots, n\}$. This walk, when started in~$0$, still reaches $n$ in an expected number of $O(n^2)$ iterations despite the complete absence of any drift in $[1..n-1]$.  The key to the analysis is to not regard the drift $E[W_{t+1} - W_{t} \mid W_{t}]$ of the process, but instead the drift of the process $(W_{t}^2)$. Then an easy calculation gives $E[W^2_{t+1} - W^2_{t} \mid W_{t} = x] = \frac 12 (x+1)^2 + \frac 12 (x-1)^2 - x^2 = 1$ for all $x \in [1..n-1]$ (see \cite[Section~5]{Goe-Koe-Kre:X:18} for an extensive discussion). 
Consequently, by regarding the drift with respect to $(W_t^2)$ instead of the original process $(W_t)$, we obtain an additive drift of $1$, and from this an expected time of $O(n^2)$ to reach state $n$. This has also been applied to the analysis of randomized search heuristics, see for example~\cite[Theorem~3.18]{Krejca19}.

Apparently more common are transformations with exponents smaller than one. \cite[Theorem~2]{Jansen07} turned a region with small drift into one with significantly more drift by employing the concave potential function $x \mapsto \sqrt{x}$. He wrote that any other function $x \mapsto x^\varepsilon$ with $\varepsilon < 1$ would be equally suitable to obtain the same tight upper bound. Essentially the same argument was used in a  more general setting in \cite{ColinDF14}. The $x \mapsto \sqrt x$ transformation was also used in the analysis how the sampling frequency of a neutral bit in a run of an EDA approaches the boundary values~\cite[Theorem~6]{DoerrZ20tec}. 

In \cite[Theorem~5]{Gie-Koe:c:14} a negative drift in a (small) part of the search space was overcome by considering random changes which make it possible for the algorithm to pass through the area of negative drift by chance. This was formalized by using a tailored potential function turning negative drift into positive drift by excessively rewarding changes towards the target, as opposed to steps away from the target. This ad-hoc argument was made formal and cast into a \emph{Headwind Drift Theorem} in \cite[Theorem~4]{Koe-Lis-Wit:c:15}.

In abstract terms, the art here is finding a potential function $g : \Z_{\ge 0} \to \R$ that transforms the unbiased process $(X_{t})$ into a process $(g(X_{t}))$ with constant drift, so that we can apply the additive drift theorem to obtain a bound of $O(g(X_0))$ on the expected optimization time. In order to obtain a positive drift, such a potential function has to be increasing and convex, and since the expected optimization time is $g(X_0)$, at the same time the potential function should increase as slowly as possible. 

For our situation, it turns out that $g$ defined by $g(x) = x \ln(x)$ is a good choice as this again gives a constant drift and thus an expected time of roughly $O(\log(1/\delta)/\delta)$ to reach a state $\Omega(1/\delta)$, from where on we will observe that also the original process has sufficient drift. We are not aware of this potential function being used so far in the theory of evolutionary algorithms (apart from a similar function being used in~\cite{AntipovDY19}, a work done in parallel to ours).

A technical annoyance in the analysis of the time taken to reach $\Omega(1/\delta)$ is that the additive drift theorem, for good reason, does not allow that the process overshoots the target. In the classical formulation, this follows from the target being $0$ and the process living in the non-negative numbers. For this reason, we cannot just show that the process $(g(X_{t}))$ has a constant drift, but we need to show this drift for a version of this process that is suitably restricted to the range $[1..\Theta(1/\delta)]$. This was a major technicality in the previous version of this work~\cite{DoerrK19}. In this version, we greatly simplify this part by using a version of the drift theorem (Theorem~\ref{thm:driftovershoot}) recently proposed by Krejca~\cite{Krejca19} that allows overshooting the target (at the price that the time bound depends not on the distance of the target, but the distance plus the expected overshooting).

Once the process has reached a value of $\Omega(1/\delta)$, the drift is strong enough to rely on making progress from the drift (and not the random fluctuations around the expectation). This is easy when the process is above $X_t = \omega(1/\delta^2)$, since then the expected progress of at least $\Omega(\delta X_t)$ is asymptotically larger than the typical random fluctuation of order $\Omega(\sqrt{X_t})$. Hence a simple Chernoff bound is enough to guarantee that each single iteration gives $X_{t+1} \ge (1-o(1)) (1+\delta) X_t$. When $X_t$ is smaller, say only $\Theta(1/\delta)$, only the combined result of $\Theta(1/\delta)$ iterations gives an expected progress large enough to admit such a strong concentration. Since the iterations are not independent, we need some careful martingale concentration arguments in this regime. Since this part is non-trivial and uses some methods that might be of broader interest, we put this into the following separate subsections. Also, we note that the specific result that the process rarely goes below half its starting point could have some independent interest (and we shall need it later again, in the proof of Theorem~\ref{thm:NewLevelBasedTheorem} to prove the level-based theorem).

\subsubsection{Additive Drift with Overshooting}

We now give a version of the additive drift theorem~\cite{He-Yao:j:01,He-Yao:j:04} as shown in ~\cite[Lemma~3.7]{Krejca19}, here slightly reformulated to best fit our purposes. In contrast to most other versions of the additive drift theorem, it allows that the process overshoots the target. This is usually implicitly forbidden by regarding processes in $\R_{\ge 0}$ and the first time to reach state~$0$. 

This extension is not very deep, but has apparently not been known too well before (as the several works that overcome the overshooting problem with hand-made methods, including~\cite{DoerrK19}, show). We note that the arguments needed to prove such a result have been known before in this community: For example, both~\cite[Lemma~12]{Jagerskupper07} and~\cite[Lemma~7]{DoerrK15} prove lower bounds for expected run times in a way that can immediately be turned into proofs for upper bounds that allow overshooting (by switching the direction of the inequality in both assumptions and results). The proof of~\cite[Lemma~2.6]{WegenerW05}, a result for hitting a particular value, can easily be extended to overshooting the value (for this, it suffices to note that $E[\sum_{i=1}^{\tau_s} D_i]$ is the value of the process after reaching or overshooting~$s$).

\begin{theorem}[Additive Drift Theorem, upper bound with overshooting]
    \label{thm:driftovershoot}
  Let $a, b \in \R$ with $a \le b$. Let $(X_t)_{t \in \N}$ be a random process over~$[-\infty,b]$. Let $T = \inf\{t \mid X_t \leq a\}$ be the first time the process reaches or drops below $a$. Suppose that there is $\delta > 0$ such that 
  \[X_t - E[X_{t+1} \mid X_0,\ldots,X_t] \geq \delta\] 
  for all $t < T$. Then
    \[
        E[T \mid X_0] \leq \frac{E[X_{(T \mid X_0)}] - X_0}{\delta}\ .
    \]
\end{theorem}

We note that the version of this result given in~\cite{Krejca19} is slightly stronger. There the condition that the process does not take values larger than some -- arbitrary -- number $b$ was replaced by the weaker condition that this only holds up to time $T$.

\subsubsection{Progress From Random Fluctuations: Creating Drift Where There is no Drift}\label{sec:nodrift}

In this subsection, we analyze how the process reaches a value of at least $D_0 = \min\{\lceil \delta/100 \rceil,n\}$. In this regime, the drift of $(X_t)$ is so low that the true reason for making progress is not the drift, but the random fluctuations stemming from the non-trivial variance. To turn these into an exploitable drift, we regard the process $(g(X_t))$ for a suitable function $g$, observe that this process has a positive drift, and use this drift to estimate the time to reach or exceed $D_0$.

We use
\begin{equation}\label{eq:defOfG}
g: \R_{\geq 0} \rightarrow \R, x \mapsto x \ln x,
\end{equation}
where, by convention, $g(0) := 0$, which renders $g$ continuous in~$0$. To establish the desired drift, we need a few technical results about~$g$. Via a Taylor expansion of $g$ around a given point $a$, we obtain the following estimates for $g$.

\begin{lemma}\label{lem:boundingXLogX}
For all $a > 0$ and $x \geq 0$, we have 
\begin{align*}
g(x) & \le a \ln a + (x-a)(1+ \ln a) + (x-a)^2 \frac{1}{a},\\ 
g(x) &\geq a \ln a + (x-a)(1+ \ln a) + (x-a)^2 \frac{1}{2a} - (x-a)^3 \frac{1}{6a^2}.
\end{align*}
\end{lemma}

\begin{proof}
Let $a > 0$ be given. We prove the (slightly more complicated) lower bound first, showing the claim for positive $x$ and then arguing with continuity.
We let $f: \R_+ \rightarrow \R$ be such that, for all $x \in \R_{+}$,
$$
f(x) = x \ln x - a \ln a - (x-a)(1+ \ln a) - (x-a)^2 \frac{1}{2a} + (x-a)^3 \frac{1}{6a^2}.
$$
Then we have, for all $x \in \R_+$,
$$
f'(x) = \ln x + 1 - (1+\ln a) - \frac{2x-2a}{2a} + \frac{3x^2-6xa+3a^2}{6a^2}
$$
and
$$
f''(x) = \frac{1}{x} - \frac{1}{a} + \frac{x}{a^2} - \frac{1}{a} = \left(\frac{1}{\sqrt{x}} - \frac{\sqrt{x}}{a} \right)^2.
$$
In particular, we have $f(a) = 0$, $f'(a) = 0$, and $f''(x) \geq 0$ for all $x \in \R_+$.
This shows that for all $x \in \R_+$, we have $f(x) \geq 0$. By the continuity of $f$, we also obtain $f(0) \geq 0$, and thus the claim.

For the upper bound, we regard $f: \R_+ \rightarrow \R$ defined by 
$$
f(x) = x \ln x - a \ln a - (x-a)(1+ \ln a) - (x-a)^2 \frac{1}{a}
$$
and compute
$$
f'(x) = \ln x + 1 - (1+\ln a) - \frac{2x-2a}{a} = \ln\left(\frac{x}{a}\right) + 2 - \frac{2x}{a}
$$
as well as 
$$
f''(x) = \frac{1}{x} - \frac{2}{a}
$$
for all $x \in \R_+$. Thus, $f''(x) > 0$ for $x < a/2$, $f''(x) = 0$ for $x = a/2$, and $f''(x) < 0$ for $x > a/2$. Consequently, $f'$ is zero for at most two arguments. Since $\lim_{x \rightarrow 0} f'(x) = - \infty = \lim_{x \rightarrow \infty} f'(x)$ and $f'(a/2) > 0$, by the intermediate value theorem there exist exactly two $x$ such that $f'(x) = 0$, one being larger than $a/2$ and the other smaller. Note that $f'(a) = 0$. From $\lim_{x \rightarrow 0} f(x) = 0$  and $f(a)=0$, the only local maximum being at $a$, we can thus conclude that $f$ is non-positive. \qed
\end{proof}

We use the estimates above to show that, under suitable circumstances, the expected $g$-value of a random variable $X$ is larger than $g(E[X])$. The lower bound in the theorem below will be used to argue that even for a process $(X_t)$ with no drift, that is, $E[X_{t+1} \mid X_t] = X_t$, the process $(g(X_t))$ has a positive drift. We need the upper bound to estimate the expected overshooting of the target when applying the additive drift theorem with overshooting (Theorem~\ref{thm:driftovershoot}).

\begin{theorem}\label{Xthm:expectedDriftInXLogX}
Let $g$ be defined as above in Equation~\eqref{eq:defOfG}. Let $X$ be a non-negative random variable with positive expectation. Let $\mu_3 = E[(X-E[X])^3]$. Then 
$$
g(E[X]) + \frac{\Var[X]}{E[X]} \ge E \left[ g(X) \right] \geq g(E[X]) + \frac{\Var[X]}{2E[X]} - \frac{\mu_3}{6E[X]^2}.
$$
\end{theorem}

\begin{proof}
We use Lemma~\ref{lem:boundingXLogX} with $a = E[X]$.\qed
\end{proof}

The following two corollaries follow immediately from the theorem above by recalling that the second and third central moments of a binomially distributed random variable $X \sim \Bin(n,p)$ are $\Var[X] = np(1-p)$ and $E[(X - E[X])^3] = np(1-p)(1-2p)$. For technical reasons, we need the first estimate also for random variables $X \sim \Bin(n,p)+ K$ for some non-negative number $K$.

\begin{corollary}\label{cor:gUB}
If $X \sim \Bin(n,p) + K$ for some $n \in \N$, $p \in (0,1]$, and $K \ge 0$, then  
$$
E \left[ g(X) \right] \leq g(E[X]) + (1-p).
$$
\end{corollary}

\begin{corollary}\label{cor:gLB}
If $X \sim \Bin(n,p)$ for some $n \in \N$ and $p \in (0,1]$, then  
$$
E \left[ g(X) \right] \geq g(E[X]) + \frac{1-p}{2} - \frac{(1-p)(1-2p)}{6E[X]}.
$$
For $p \geq 1/n$, this yields
\begin{equation}
E \left[ g(X) \right] \geq g(E[X]) + \frac{1-p}{3}.\label{eq:g}
\end{equation}
\end{corollary}

We are now prepared to show the following result.

\begin{lemma}\label{lem:nodrift}
Let $(X_{t})_{t \in \N}$ be a stochastic process over the positive integers. Assume that there are $D_0, k \in \Z_{\geq 1}$ and $\gamma_0 < 1$ such that $D_0-1 \le \gamma_0 k$ and for all $t \ge 0$ and all $x \in [1..D_0-1]$ with $\Pr[X_{t} = x] > 0$ we have the unbiased binomial condition
\begin{description}
	\item[(Bin$_0$)] $(X_{t+1} \mid X_{t} = x) \succeq \mathrm{Bin}(k, x/k)$.
\end{description}
Let $T := \min\{t \geq 0 \mid X_{t} \geq D_0\}$. Then 
\[
E[T] \le \frac{6 D_0 \ln(2 D_0)}{{1-\gamma_0}}.
\]
\end{lemma}

\begin{proof}
There is nothing to show for $D_0=1$, so we assume $D_0 \ge 2$ in the remainder. 
For technical reasons, let us regard the process $(X'_t)$, which agrees with $(X_t)$ while not larger than $D_0$, but follows the pessimistic law $X'_{t+1} \sim \Bin(k, X_t / k)$ in the iteration where $D_0$ is exceeded. More precisely, we let $X'_0 = X_0$. Given that some $X'_t$ is defined already, we define $X'_{t+1}$ as follows. If $X'_t \le D_0$, then for all $x \ge 1$ we have
\begin{align*}
  \Pr[X'_{t+1} = x] = 
    \begin{cases} 
      \Pr[X_{t+1} = x], &\mbox{ if $x < D_0$,}\\
      \Pr[\Bin(k, X_t/k) = x], &\mbox{ if $x > D_0$,}
    \end{cases}
\end{align*}
and the remaining probability mass is put on $D_0$, that is, 
\[
\Pr[X'_{t+1} = D_0] = 1 - \sum_{x=1}^{D_0 -1} \Pr[X_{t+1} = x] - \! \sum_{x=D_0+1}^k \Pr[\Bin(k, X_t/k) = x].\] 
If $X'_t > D_0$, we let $X'_{t+1} = X'_t$ with probability one. Since the process $(X'_t)$ agrees with $(X_t)$ while less than $D_0$, we have $T' := \min\{t \mid X'_t \ge D_0\} = \min\{t \mid X_t \ge D_0\} =: T$.  

We estimate $T'$. Consider some time $t$ such that $x := X'_t$ is in $[1..D_0-1]$. Let $Y \sim \Bin(k,x/k)$. Since $X'_{t+1} \succeq Y$ and $g$ is monotonically increasing in $\{0\} \cup [1, \infty)$, we have $E[g(X'_{t+1})] \ge E[g(Y)]$. By Equation~\eqref{eq:g} in Corollary~\ref{cor:gLB}, we have \[E[g(Y)] \ge g(E[Y]) + \frac{1-(D_0-1)/k}{3} \ge g(x) + \frac{1-\gamma_0}{3}.\]
Consequently, we have $E[g(X'_{t+1}) - g(X'_t) \mid X'_t < D_0] \ge \frac{1-\gamma_0}{3}$. 

To apply the additive drift theorem with overshooting (Theorem~\ref{thm:driftovershoot}), we observe that $T = T' = \min\{t \mid g(X'_t) \ge D_0 \ln(D_0)\}$ and compute $E[g(X'_T)]$. By construction, $X'_T \sim (Y \mid Y \ge D_0)$ for some $Y$ following a binomial law with parameters $k$ and some $p \le (D_0-1)/k$. By elementary arguments analogous to those used in the proof of~\cite[Lemma~1]{DoerrD18}, see also \cite[Lemma~1.7.3]{Doerr20bookchapter} and the comment following its proof, $X'_T$ is stochastically dominated by $D_0 + \Bin(k,(D_0-1)/k)$, which immediately gives
$$
E[X'_T] \leq 2D_0 - 1.
$$ By Corollary~\ref{cor:gUB}, we have $E[g(X'_T)] \le (2 D_0-1) \ln(2 D_0-1) + 1 \le (2 D_0-1) \ln(2 D_0) + 1 \le 2D_0 \ln(2 D_0)$, the last estimate using $D_0 \ge 2$. Consequently, the additive drift theorem with overshooting gives
\begin{equation*}
E[T'] \le \frac{2 D_0 \ln(2 D_0)}{\frac{1-\gamma_0}{3}}. 
\end{equation*} \qed
\end{proof}

We remark that, in principle, Lemma~\ref{lem:nodrift} can be strengthened by taking into account the starting point $X_0$. Assuming for simplicity that $X_0$ takes only values in $[1..D_0]$, this would give a result like 
\[E[T'] \le \frac{2 D_0 \ln(2 D_0) - E[g(X_0)]}{\frac{1-\gamma_0}{3}} \le \frac{2 D_0 \ln(2 D_0) - g(E[X_0])}{\frac{1-\gamma_0}{3}},\]
where the last estimate stems from the convexity of $g$ and Jensen's inequality. Since $E[g(X_0)]$ is at most $D_0 \ln(D_0)$, we gain at most a constant factor in the estimate of $E[T']$. The reason for this weak improvement is that we estimated $E[g(X'_T)]$ very coarsely. However, even with a better estimate of $E[g(X'_t)]$, asymptotically stronger results would only be possible in the case that $X_0$ is very close to~$D_0$, that is, that $D_0 \ln(D_0) - E[g(X_0)] = o(D_0 \ln(D_0))$, which we do not expect in our typical applications. 

We note further that the problem of overshooting and the resulting negative impact on the hitting time estimate is real. Even if $X_0 = D_0 - 1$ with probability one, we see that when taking $k = 2(D_0-1)$ for simplicity, we have $X_1 \le D_0 - \Omega(\sqrt D_0)$ with constant probability (that this is possible for an unbiased process stems from the fact that $X_1$ overshoots $D_0$ by a comparable amount). We omit a formal proof, but note that from $X_1 \le D_0 - \Omega(\sqrt D_0)$, the process takes an expected number of $\Omega(\sqrt D_0)$ iterations to reach or overshoot~$D_0$.

\subsubsection{Submartingale Arguments Proving A Steady Progress From $D_0$ on}

In this subsection, we shall prove that once a process satisfying the assumptions of Theorem~\ref{thm:first} has reached a value of $D \ge D_0 := \min\{\lceil 100/\delta \rceil, n\}$, it usually makes a steady progress of a constant factor increase in $\Theta(1/\delta)$ iterations without ever going below $D/2$. To show this result, we use a submartingale argument that might prove to be useful in other analyses of evolutionary algorithms as well. We build on the following result from Freedman~\cite[Theorem~4.1]{Fre:j:75}, cited in a more compact manner in~\cite[Theorem~A]{FanGraLiu:j:15} (adjusted for submartingales rather than supermartingales).

\begin{theorem}\label{thm:subm}
  Let $S_0, S_1, \ldots, S_N$ be a submartingale sequence, that is, we have $E[S_\ell - S_{\ell-1} \mid S_0, \ldots, S_{\ell-1}] \ge 0$ for all $\ell \in [1..N]$. Let $\langle S \rangle_m := \sum_{\ell=1}^m E[(S_\ell - S_{\ell-1})^2 \mid S_0, \dots, S_{\ell-1}]$ for all $m \in [1..N]$. Assume that $S_\ell - S_{\ell-1} \ge -u$ with probability $1$ for all $\ell \in [1..N]$. Then for all $\lambda, \nu > 0$, 
  \[\Pr[\exists m \in [1..N] : S_m \le -\lambda \wedge \langle S \rangle_m \le \nu^2] \le \exp\left(-\frac{\lambda^2}{2(\nu^2 + \lambda u)}\right).\]
\end{theorem}

We use this result to bound the probability that the process started in $D \ge D_0$ at time $t$ at any time $s \in [t..t+O(1/\delta)]$ goes below $D/2 + (s-t)\delta D/2$. If this does not happen, we in particular have $X_{t + \lceil 3/\delta \rceil} \ge 2D$.

\begin{lemma}\label{lem:subm}
  Let $(X_{t})_{t \in \N}$ be a stochastic process over the positive integers. Assume that there are $n,k \in \Z_{\geq 1}$ and $\delta \in (0,1]$ such that $n-1 \le (1+\delta)^{-1} k$ and for all $t \ge 0$ and all $x \in [1..n-1]$ with $\Pr[X_{t} = x] > 0$ we have the binomial condition
\begin{description}
	\item[(Bin)] $(X_{t+1} \mid X_{t} = x) \succeq \mathrm{Bin}(k,(1+\delta) x/k)$.
\end{description}
Let $t_0 \geq 0$ and $100 / \delta \le D < n$ such that $X_{t_0} = D$. Let $\tilde n = \min\{n, 2D\}$, $T = \min\{t \mid X_t \ge \tilde n\}$, and $T_1 = \min\{T, t_0 + \lceil 3/\delta \rceil\}$. Then 
  \[\Pr[\exists s \in [t_0..T_1] : X_s \leq \tfrac 12 D + \tfrac 12 (s-t_0)\delta D] \le \exp\left(- \delta D / 169 \right).\]
\end{lemma}

Since the proof of this lemma is not obvious, let us describe the main ideas before stating the formal proof. From \textbf{(Bin)} we immediately see that $(X_t)$ is a submartingale. However, since each $X_t$ may take all values in $[1..k]$ with positive probability, this submartingale does not admit good absolute bounds on the submartingale differences (the variable $u$ in Theorem~\ref{thm:subm}). 

For this reason, we write the variable $X_t$ as a sum of the $k$ independent binary random variables $Y_{t1}, \dots, Y_{tk}$ that describe the binomial distribution in~\textbf{(Bin)}. Then the, suitably defined, submartingale differences $Y_{tj} - \frac 1k X_{t-1}$ define a submartingale with differences bounded by one (we can take $u = 1$). 

To make the progress of this submartingale visible, we would like to regard instead the submartingale with differences $Y_{ti} - \frac 1k X_{t-1} - D\delta/(2k)$. If $X_{t-1} \ge D/2$, this difference still has a non-negative expectation (as necessary for a submartingale). Since we cannot rule out that some $X_{t-1}$ is less than $D/2$, we define our submartingale via the differences $Z_{tj} = Y_{tj} - \frac 1k X_{t-1} - \Delta_t$, where $\Delta_t = D\delta/(2k)$ when $X_{t-1} \ge D/2$ and $\Delta_t = 0$ otherwise. This defines a submartingale. Via Theorem~\ref{thm:subm}, we shall show that with high probability, this submartingale never goes below $D/2$. This in particular implies that all $X_t$ are at least $D/2$, and hence, that all $\Delta_t$ are $D\delta/(2k)$. Consequently, $X_t$ is not only at least $D/2$, but it is even at least $D/2 + tD\delta/2$, which shows the desired progress.

\begin{proof}
To ease the notation, let us assume that $t_0 = 0$. 

To have a better control over the one-step variances to be computed later, 
we first argue that we can pessimistically assume that the progress is exactly the one described by the binomial distributions in \textbf{(Bin)}. More precisely, let $X'_0 = X_0$ and define recursively $X'_t$ as a random variable with distribution $X'_t \sim \Bin(k,(1+\delta)X'_{t-1}/k)$ for all $t \ge 1$ such that ${X_{t-1} < \tilde n}$. Then a simple induction shows that $X_t \succeq X'_t$ for all $t \in \N_0$: If $X'_{t-1} \preceq X_{t-1}$, then $X_t \succeq \Bin(k (1+\delta), X_{t-1} / k) \succeq \Bin(k (1+\delta), X'_{t-1} / k) \sim X'_{t}$. Consequently, 
\[T = \min\{t \mid X_t \ge \tilde n\} \le \min\{t \mid X'_t \ge \tilde n\} =: T'.\] 
To show our claim it thus suffices to show 
\begin{equation}
\Pr[\exists s \in [0..\min\{T', \lceil 3/\delta \rceil\}] : X'_s \leq \tfrac 12 D + \tfrac 12 s\delta D] \le \exp\left(- \delta D / 169 \right). \label{eq:newclaim}
\end{equation}

To ease the argument, let us artificially continue the process in case it reaches a state of at least $\tilde n$ before time $\lceil 3/\delta \rceil$. More specifically, let 
\begin{description}
	\item[(Bin')] $X'_{t+1} = X'_t + 2D\delta$ with probability one when $X'_t \ge \tilde n$. 
\end{description}
Then this modified process agrees with the original process $(X')$ up to time $T'$, and thus it satisfies~\eqref{eq:newclaim} if and only if the original process does. We can thus work with the modified process in the following. 

We define random variables $(Y_{tj})_{1 \le t \leq \lceil 3/\delta \rceil, 1 \leq j \leq k}$ as follows. If $X'_{t-1} < \tilde n$, then $Y_{t1}, \dots, Y_{tk}$ are independent Bernoulli random variables with success probability $p_t = (1+\delta) X'_{t-1} /k$. Note that $p_t \le (1+\delta) \tilde n / k \le (1+\delta) 2D / k \le 4D/k$. If $X'_{t-1} \ge \tilde n$, then $Y_{tj} = (X'_{t-1} + 2D\delta) / k$ with probability one for all $j$, and we set $p_t = 0$. By \textbf{(Bin)} and \textbf{(Bin')}, we can assume $X'_t = \sum_{j=1}^k Y_{tj}$ in either case.

Further, for all $t \ge 1$, we define $\Delta_t \in \{0,D\delta/(2k)\}$ as follows. If $X'_{t-1} \ge D/2$, then $\Delta_t = D\delta / (2k)$; otherwise, let $\Delta_t = 0$. Define $Z_{tj} = Y_{tj} - \frac 1k X'_{t-1} - \Delta_t$ for all $t, j$. By definition of $\Delta_t$, we have 
\begin{equation}
E[Z_{tj}] = 
\begin{cases}
1.5 D \delta / k & \text{if } X'_{t-1} \ge \tilde n\\
\delta X'_{t-1}/k - D\delta/(2k) \in [0,1.5 D \delta / k] & \text{if } D/2 \le X'_{t-1} \le \tilde n\\
\delta X'_{t-1}/k & \text{if } X'_{t-1} \in [0,D/2).
\end{cases}\label{eq:ez}
\end{equation}
In particular, 
\begin{equation}
  E[Z_{tj}] \ge 0. \label{eq:subm}
\end{equation}
We trivially observe
\begin{equation}
  Z_{tj} \ge -(1+\tfrac{D\delta}{2k}) \ge -1.25, \label{eq:subm2}
\end{equation}
where in the second estimate we used that $D \le n-1 \le (1+\delta)^{-1} k$ and that the term $\delta / (1+\delta)$ has a unique maximum at $\delta = 1$ in $[0,1]$. Finally, using again $D \le (1+\delta)^{-1} k$ and noting that the term $\delta^2 / (1+\delta)$ is maximal for $\delta = 1$ (assuming $\delta \le 1$), with Equation~\eqref{eq:ez} we compute that, conditional on $X_{t-1}$,
\begin{align}
  E[Z_{tj}^2] &= \Var[Z_{tj}] + E[Z_{tj}]^2\nonumber\\
  &\le p_t(1-p_t) + (1.5 D\delta / k)^2 \nonumber\\
  &\le (D/k) (4 + 2.25 D\delta^2 / k) \le 5.125 D/k. \label{eq:subm3}
\end{align}

For all $i \geq 0$ and $j \in [0..k-1]$ let further
\[
S_{ik+j} := \sum_{j' = 1}^j Z_{i+1,j'} + \sum_{i' = 0}^{i-1} \sum_{j'=1}^{k} Z_{i'+1,j'},
\]
that is, the sum of the first $ik+j$ $Z$-variables in the natural lexicographic ordering.
By Equation~\eqref{eq:subm}, this is a submartingale, by Equation~\eqref{eq:subm2}, we have $S_\ell - S_{\ell-1} \ge -(1+\frac{D\delta}{2k}) \ge -1.25 =: -u$ for all $\ell$, and by Equation~\eqref{eq:subm3}, we have $E[(S_\ell - S_{\ell-1})^2] \le 5.125 D/k =: \nu^2_0$ regardless of $S_1, \dots, S_{\ell-1}$. Consequently, we may apply Theorem~\ref{thm:subm} with $N = \lceil 3/\delta \rceil k$, $\lambda = D/2$, and $\nu^2 = N \nu^2_0 = 5.125 D \lceil 3/\delta \rceil \le 5.125 D (4 / \delta) = 20.5 D / \delta$, and obtain
\begin{align*}
  \Pr[\exists &m \in [1..N] : S_m \le -\lambda] \\
  &= \Pr[\exists m \in [1..N] : S_m \le -\lambda \wedge \langle S \rangle_m \le \nu^2] \\
  &\le \exp\left(-\frac{\lambda^2}{2(\nu^2 + \lambda u)}\right) \le \exp\left(-\frac{D^2}{8(20.5 D / \delta + 0.625 D)}\right) \\
  & \le \exp(-\delta D / 169).
\end{align*}
   
Let us assume that this rare event does not occur, that is, we have $S_m \ge - D/2$ for all $m \in [0..N]$. We note that when $m$ is a multiple of~$k$, then $S_m = X_{m/k} - X_0 - k \sum_{s = 1}^{m/k} \Delta_s$. With $X_0 = D$ and $S_m \ge -D/2$, we obtain $X_t \ge -D/2 + D + k \sum_{s = 1}^t \Delta_s$ for all $t$. Consequently, we have $X_t \ge D/2$ for all $t$, hence $\Delta_t = D\delta/(2k)$ for all $t$, and thus $X_t \ge D/2 + tD\delta/2$ as desired.\qed
\end{proof}

With an iterated application of the previous result, we can show that the process has a decent chance to reach the target $n$ in time $O(\log(n) / \delta)$. We shall later only need the result with the success probability $0.2782$, but since we easily prove a stronger bound for larger starting points $D$ and since such results might be useful in other contexts, we also prove such an estimate.

\begin{lemma}\label{lem:climbUp}
  Let $(X_{t})_{t \in \N}$ be a stochastic process over the positive integers. Assume that there are $n,k \in \Z_{\geq 1}$ and $\delta \in (0,1]$ such that $n-1 \le (1+\delta)^{-1} k$ and for all $t \ge 0$ and all $x \in [1..n-1]$ with $\Pr[X_{t} = x] > 0$ we have the binomial condition
\begin{description}
	\item[(Bin)] $(X_{t+1} \mid X_{t} = x) \succeq \mathrm{Bin}(k,(1+\delta) x/k)$.
\end{description}
Let $t_0 \geq 0$ and $100 / \delta \le D < n$ such that $X_{t_0} = D$. Then, with probability at least $\max\{0.2782, 1 - \frac{1}{\exp(\delta D / 169) - 1}\}$, the process reaches or exceeds $n$ within at most $\lceil \log_2(n/D)\rceil \lceil 3 / \delta \rceil$ iterations.
\end{lemma}

\begin{proof} 
Using Lemma~\ref{lem:subm}, with probability at least $1 - \exp(- \delta X_0 / 169) = 1 - \exp(- \delta D / 169)$ there is $t_1 \leq t_0 + \lceil 3/\delta \rceil$ such that $X_{t_1} \geq \min\{2D, n\}$. Given this event and assuming $X_{t_1} < n$, with probability at least $1 - \exp(- \delta X_{t_1} / 169) \ge 1 - \exp(- 2 \delta D / 169)$, there is a $t_2 \leq t_1 + \lceil 3/\delta \rceil$ such that $X_{t_2} \ge \min\{2 X_{t_1},n\} \geq \min\{4D,n\}$. Repeating this doubling argument at most $\lceil \log_2(n/D) \rceil$ times, we obtain a state of at least $n$. This takes at most $\lceil \log_2(n/D)\rceil \lceil 3 / \delta \rceil$ iterations and works out as desired with probability at least
\begin{align*}
  \prod_{i=0}^\infty &\left(1 - \exp(- 2^i \delta D / 169)\right) \\
  &\ge 1 - \sum_{i=0}^\infty \exp(-2^i \delta D / 169)\\
  &\ge 1 - \sum_{i=0}^\infty \exp(-(i+1) \delta D / 169)\\
  &= 1 - \frac{1}{\exp(\delta D / 169) - 1},
\end{align*}
where the first inequality follows from a Weierstass product inequality, a mild extension of Bernoulli's inequality (see, e.g., \cite[Lemma~1.4.8]{Doerr20bookchapter}), and the last equation computes the geometric series. When $D$ is small, this estimate can be negative and then is not very useful. For this case, using our assumption that $D \ge 100/\delta$, we compute
\begin{align*}
  \prod_{i=0}^\infty &\left(1 - \exp(- 2^i \delta D / 169)\right) \\
  & \ge \prod_{i=0}^\infty \left(1 - \exp(-2^i \cdot \tfrac{100}{169})\right) \\
  &\ge \left(\prod_{i=0}^3 \left(1 - \exp(-2^i \cdot \tfrac{100}{169})\right) \right) \cdot \left(1 - \sum_{i=4}^\infty \exp(-2^i \cdot \tfrac{100}{169})\right)\\
  &\ge 0.2783 \cdot \left(1 - \exp(-\tfrac{1600}{169}) \sum_{i=0}^\infty \exp(-\tfrac{100}{169})^i\right) \\
  &= 0.2783 \cdot \left(1 - \exp(-\tfrac{1600}{169}) \frac{1}{1 - \exp(-\tfrac{100}{169})}\right) \ge 0.2782,
\end{align*}
which gives the desired bound.\qed
\end{proof}

\subsubsection{Proof of Theorem~\ref{thm:first} for $\delta \le 1$}\label{ssec:prooffirst}

By combining the two main insights of the two preceding subsections, we now prove the first up-drift theorem in the case $\delta \le 1$. Since the proof uses a result known as Wald's equation~\cite{Wald44}, we first state a simplified version of this result.

\begin{theorem}[Wald's equation]\label{thm:wald}
  Let $M \in \R$. Let $X_1, X_2, \dots$ be an infinite sequence of non-negative random variables with $E[X_i] \le M$ for all $i \in \N$. Let $T$ be a positive integer random variable with $E[T] < \infty$. Assume that for all $t \in \N$, we have $E[X_t \mathbf{1}_{\{T \ge t\}}] = E[X_t] \Pr[T \ge t]$. Then 
  \[E\left[\sum_{t = 1}^T X_t\right] \le M \cdot E[T].\]
\end{theorem}

We now give the proof of the first up-drift theorem for the case that $\delta \le 1$.

\begin{proof}[of Theorem~\ref{thm:first} for $\delta \le 1$]
Let us call a \emph{phase} of the process $(X_t)$ the time interval used to first reach a value of at least $D_0 = \min\{\lceil 100/\delta \rceil, n\}$ and then another $T_2 = \lceil \log_2(n/D_0)\rceil \lceil 3 / \delta \rceil \leq \log_2(n) \lceil 3 / \delta \rceil$ iterations. By Lemma~\ref{lem:nodrift}, the expected time to reach a value of at least $D_0 = \min\{\lceil 100/\delta \rceil, n\}$ is at most $T_1 = \frac{6 D_0 \ln(2 D_0)}{{1-\gamma_0}}$. Hence a phase has an expected length of at most $M = T_1 + T_2$. 

By Lemma~\ref{lem:climbUp}, a phase is successful, that is, reaches or exceeds $n$, with probability at least $0.2782$. Hence the number of phases until a successful one is encountered, is described by a geometric random variable $T$ with success rate $p = 0.2782$.

Hence by Wald's equation (Theorem~\ref{thm:wald}), the expected time to reach or exceed $n$ therefor is at most
\[
M E[T] = M \frac 1p \le \frac{1}{0.2782}\left(\frac{6 D_0 \ln(2 D_0)}{{1-\gamma_0}} + \log_2(n) \lceil 3 / \delta \rceil\right).
\]
The claim follows from noting that $1/0.2782 < 3.6$. \qed
\end{proof}

\subsubsection{The Case $\delta > 1$}\label{ssec:deltalarge}

In this section, we treat the case that $\delta$ is larger than one. In this case, the up-drift is so strong that we do not have a significant phase in which the progress stems mostly from random fluctuations. Rather, we can argue that with constant ``success'' probability, the process increases by a factor of at least $(1+\delta/2)$ in each iteration and thus reaches the target of $n$ in at most $\lceil \log_{1+\delta/2}(n) \rceil = O(\log(n)/\log(\delta))$ iterations. In case of failure, a simple restart argument (leading to an expected constant number of restarts of the argument) suffices to show the same bound for the expected time to reach a state of at least $n$. 

This argument alone would give a relatively low success probability of 
\[\prod_{i=0}^\infty (1 - \exp(-(1+0.5)^i /16)) \le 3.4 \cdot 10^{-6}\] when proceeding as in the proof below, using $\delta = 1$, and estimating this infinite product via its first ten factors. Consequently,  a very high implicit constant in the $O(\log(n)/\log(\delta))$ bound would result. To overcome this, we first argue that it takes at most an expected number of $62$ iterations to reach a state of at least $32$. From this point on, the probability to increase by a factor of $(1 + \delta/2)$ in each subsequent iteration is more than  $0.78$. While we did not aim at obtaining the best possible constants, we decided to follow this line of argument to obtain a leading constant that is not only of theoretical interest. We note that the same argument could be used with intermediate targets larger than $32$ and increase factors closer to $(1+\delta)$, which shows that the right asymptotics is $(1+o(1)) \log_{1+\delta}(n)$.

To prove the case $\delta > 1$ of the first up-drift theorem, we show the following lemma. It contains a statement on making less progress than expected which is stronger than what we need here, but which might be useful in other contexts.

\begin{lemma}[First Up-Drift Theorem, $\delta > 1$]
\label{lem:deltalarge}
Let $(X_{t})_{t \in \N}$ be a stochastic process over the positive integers. Assume that there are $n,k \in \Z_{\geq 1}$ and $\delta \ge 1$ such that $n-1 \le (1+\delta)^{-1} k$ and for all $t \ge 0$ and all $x \in [1..n-1]$ with $\Pr[X_{t} = x] > 0$ we have the binomial condition
\begin{description}
	\item[(Bin)] $(X_{t+1} \mid X_{t} = x) \succeq \mathrm{Bin}(k,(1+\delta) x/k)$.
\end{description}
Let $T := \min\{t \geq 0 \mid X_{t} \geq n\}$. Then 
$$
E[T] \le 2.6 \log_{1+\delta}(n) + 81.
$$
In addition, once the process has reached some state $x$ or higher, the probability to have a step with $X_{t+1} < (1+\delta/2) X_t$ before reaching $X_t \ge n$ is at most $\tfrac{1}{e^{x/32}(e^{x/32}-1)}$. In particular, once the process has reached $x = 32$, the probability to ever go below $32$ (before reaching $n$) is less than $0.22$.
\end{lemma}

\begin{proof}
  To ease the argument, we shall now assume that we have ${X_{t+1} \sim \max\{1, \Bin(2(1+\delta)X_t, 1/2)\}}$ when $X_t \ge n$. This artificial continuation of the process (similar to the one we used in Lemma~\ref{lem:nodrift}) does not change the first time to reach or overshoot the target~$n$, but allows us to disregard whether the process has reached the target earlier than thought. 
  
  We analyze one \emph{phase} of the process, started at some time $t_0$ with an arbitrary value $X_{t_0}$. We say that this phase ends (after $\ell$ iterations) when either (i)~$t_0+\ell$ is the first time not earlier than $t_0$ that $X_{t_0 + \ell} \ge n$ (``success''), or (ii)~$t_0+\ell$ is the first time such that $X_{t_0+\ell} < (1+\delta/2) X_{t_0+\ell-1}$ and $X_{t_0+\ell-1} \ge 32$ (``failure''). In simple words, the phase ends when the target is reached or when we fail to obtain a factor-$(1+\delta/2)$ increase from a state that is at least $32$.
  
  We first compute a simple upper bound for the expected length of a phase, which is valid regardless of whether we condition on success or failure. We start by estimating the expected time to reach a value of at least $32$. Since $\delta > 1$, at any time $t$ the state $X_{t+1}$ dominates a binomial distribution with expectation $2 X_t$. By the well-known fact that the median of a binomial distribution with integral expectation is equal to this expectation, first explicitly shown in~\cite{Neumann66}, we have $\Pr[X_{t+1} \ge 2 X_t] \ge 1/2$. Consequently, the time to reach $32$ is at most the time $\tilde T$ it takes for a sequence of random bits to encounter five successive ones. We note that the expectation of $\tilde T$ satisfies the recurrence $E[\tilde T] = \frac 12 (1 + E[\tilde T]) + \frac 14 (2 + E[\tilde T]) + \frac 18 (3 + E[\tilde T]) + \frac 1 {16} (4 + E[\tilde T]) + \frac 1 {32} (5 + E[\tilde T]) + \frac 1 {32} \cdot 5$, which gives $E[\tilde T] = 62$.
  
  Once a state of $32$ or more is reached, we either witness a failure or an increase by a factor of $(1+\delta/2)$. Consequently, after another $\lceil\log_{1+\delta/2}(n/32)\rceil$ iterations, we have encountered a failure or reached the target, and hence the phase has ended within this timespan. In summary, the expected length of a phase, regardless of the starting state and regardless of whether it is successful or not, is at most $62+\lceil\log_{1+\delta/2}(n/32)\rceil$ iterations. Noting that $1+\delta \leq (1+\delta/2)^2$, we have 
\begin{equation}\label{eq:deltaHalfInLogBase}
\log_{1+\delta/2}(x) \leq 2 \log_{1+\delta}(x)
\end{equation}	
for all $x \ge 1$, and thus the expected length of a phase is at most
  \[62+\lceil\log_{1+\delta/2}(n/32)\rceil \le 63 + 2 \log_{1+\delta}(n).\]

From the ``in particular'' case of the ``in addition clause'', which we shall prove shortly, we see that a phase is successful with probability at least $0.78$. By elementary properties of the geometric distribution, there is an expected number of at most $\frac{1}{0.78} \le 1.283$ phases until the process is successful, and hence reaches the target. Since each phase takes an expected number of at most $63 + 2 \log_{1+\delta}(n)$ iterations, the desired expected hitting time is at most $\frac{1}{0.78} (63 + 2 \log_{1+\delta}(n)) \le 81 + 2.6 \log_{1+\delta}(n)$ by Wald's equation (Theorem~\ref{thm:wald}). 
  
  We now prove the ``in addition'' statement. For any time $t$, by a simple Chernoff bound (e.g., Theorem~1.10.5, Equation~(1.10.12), in~\cite{Doerr20bookchapter}), we have (using $\delta \geq 1$)
  \begin{align*}
  \Pr[&X_{t+1} < (1+\delta/2) X_t] \\
  &\le \Pr[X_{t+1} < 0.75(1+\delta) X_t] \le \Pr[X_{t+1} < 0.75 E[X_{t+1}]] \\
  &\le \exp(-E[X_{t+1}]/32) \le \exp(-(1+\delta)X_t/32) \le \exp(-X_t/16).
  \end{align*}
  Assume that at some time $t_1$ we have $X_{t_1} = x$. Let us now, minimally modifying the previously introduced notation, speak of a failure when for some $t \ge t_1$ we have $X_{t+1} < (1 + \delta/2) X_t$. Noting that no failure for $i$ iterations leads to a state $X_{t_1+i} \ge (1+\delta/2)^i X_{t_1} = (1+\delta/2)^i x$, we see that the probability that no failure happens in any iteration later than $t_1$ is at least
  \begin{align*}
  \prod_{i=0}^\infty &(1 - \exp(-(1+\delta/2)^i x/16)) \\
  & \ge \prod_{i=0}^\infty (1 - \exp(-2 \cdot (3/2)^i \cdot x/32 ) \\
  & \ge 1 - \sum_{i=0}^\infty \exp(-2 \cdot (3/2)^i  \cdot x/32) \\
  & \ge 1 - \sum_{i=2}^\infty \exp(-i \cdot x/32) = 1 - \tfrac{1}{e^{x/32}(e^{x/32}-1)},
  \end{align*} 
where, similarly as in Lemma~\ref{lem:climbUp}, we employ the Weierstrass product inequality and the fact that $2 \cdot (3/2)^i \ge i+2$ for all non-negative integers~$i$. We note that for $x=32$, this bound is less than $0.22$ and the event ``no failure'' implies the event to never go below $32$. \qed
\end{proof}

\subsection{Processes That Can Reach Zero}
 
We now extend the multiplicative up-drift theorem to include state~$0$. Since the subprocess consisting only of states greater than $0$ satisfies the assumptions of the first up-drift theorem, we obtain from the latter an upper bound on the time spend above $0$. It therefore remains to estimate the time spent in state~$0$, which in particular means estimating how often the process reaches this state. In the technically more demanding case that $\delta \le 1$, we exploit that the process is a submartingale. We can thus employ the optional stopping theorem to estimate that with probability $1 - \Omega(\delta)$ the process reaches $0$ before reaching $D_0 = \min\{\lceil 100 / \delta \rceil, n\}$. Consequently, after an expected number of $O(\delta)$ attempts, the process reaches~$D_0$, and from there with constant probability never goes back to zero.

\begin{theorem}[Second Multiplicative Up-Drift Theorem]\label{thm:reverseDriftWithZeroBinomial}\label{thm:second}
Let $(X_{t})_{t \in \N}$ be a stochastic process over  $\Z_{\geq 0}$. Let $n,k \in \Z_{\geq 1}$, $E_0 > 0$, $\gamma_0 < 1$, and $\delta > 0$ such that $n - 1  \le \min\{\gamma_0 k, (1+\delta)^{-1} k\}$. Let $D_0 = \min \{\lceil 100/\delta \rceil,n\}$ when $\delta \le 1$ and $D_0 = \min \{32,n\}$ otherwise. Assume that for all $t \ge 0$ and all $x \in [0..n-1]$ with $\Pr[X_{t} = x] > 0$, the following two properties hold.
\begin{description}
	\item[(Bin)] If $x \ge 1$, then $(X_{t+1} \mid X_{t} = x) \succeq \mathrm{Bin}(k,(1+\delta) x/k)$.
	\item[(0)] $E[ \min\{X_{t+1}, D_0\} \mid X_{t} = 0] \ge E_0$.
\end{description}
Let $T := \min\{t \geq 0 \mid X_{t} \geq n\}$. Then, if $\delta \leq 1$,
\begin{align*}
E[T] 
 & \leq \frac{4D_0 }{0.2782 E_0} + \frac{21.6}{1-\gamma_0} D_0 \ln(2 D_0)+ 3.6 \log_2(n) \lceil 3 / \delta \rceil.
\end{align*}
In particular, when $\gamma_0$ is bounded away from $1$ by a constant, then $E[T] = O(\frac{1}{E_0\delta} + \frac{\log (n)}{\delta})$, where the asymptotic notation refers to $n$ tending to infinity and where $\delta=\delta(n)$ may be a function of $n$.
Furthermore, if $n > 100/\delta$, then we also have that once the process has reached state of at least $100/\delta$, the probability to ever return to a state of at most $50/\delta$ is at most $0.7218$.

If $\delta > 1$, then we have
\begin{align*}
E[T] 
 & \leq \frac{128}{0.78 E_0} + 2.6 \log_{1+\delta}(n) + 81\\
 & = O\left(\frac{1}{E_0} + \frac{\log (n)}{\log(\delta)}\right).
\end{align*}
In addition, once the process has reached state $32$ or higher, the probability to ever return to a state lower than $32$ is at most $\tfrac{1}{e(e-1)} < 0.22$.
\end{theorem}

We show the theorem by considering two different kinds of steps of the process: those spent in state~$0$ and those spent in other states. For the latter we understand what happens from Theorem~\ref{thm:first}, so it remains to see what happens in state~$0$. There are in turn two ways in which the process can be in state~$0$. Either it could have been in state~$0$ before; in this case we will use \textbf{(0)} to see how the process gets out again. More complicated is the case of returning to state~$0$.

From Theorem~\ref{thm:first} we know that it is unlikely to return back to $0$ after having reached a sufficiently high value. In order to compute a good bound on the return probability for smaller values of the process, we use the optional stopping theorem, which we state next for convenience. We use a version given by Grimmett and Stirzaker~\cite[Chapter~$12.5$, Theorem~$9$]{grimmett2001probability} that can be extended to super- and submartingales.


\begin{theorem}[\textrm{Optional Stopping}]
    \label{thm:optionalStopping}
    Let $(X_t)_{t \in \N}$ be a random process over~$\R$, and let $T$ be a stopping time\footnote{Intuitively, for the natural filtration, a stopping time~$T$ is a random variable over~$\N$ such that, for all $t \in \N$, the event $\{t \leq T\}$ is only dependent on $X_0, \ldots, X_t$.} for $(X_t)_{t \in \N}$. Suppose that
    \begin{enumerate}[label=(\alph*)]
        \item\label{item:optStopFiniteExpectation} $E[T] < \infty$ and that
        
        \item there is some value $c \geq 0$ such that, for all $t < T$, it holds that $E[|X_{t+1} - X_t| \mid X_0,\ldots,X_t] \leq c$.
    \end{enumerate}
    Then the following two statements hold.
    \begin{enumerate}
        \item\label{item:optionalStoppingSupermartingale} If, for all $t < T$, $X_t - E[X_{t+1} \mid X_0,\ldots,X_t] \geq 0$, then $E[X_T] \leq E[X_0]$.
        
        \item\label{item:optionalStoppingSubmartingale} If, for all $t < T$, $X_t - E[X_{t+1} \mid X_0,\ldots,X_t] \leq 0$, then $E[X_T] \geq E[X_0]$.
    \end{enumerate}
\end{theorem}

For the application of the optional stopping theorem it will be necessary to have a good bound on the value of the process after exceeding some value. Since no good bounds are guaranteed for the original process, we instead analyze a slightly different process which we can construct with the following lemma. It states, roughly, that we can replace a binomial random variable with expectation $E$ with a random variable that is identically distributed in $[0..E]$ and takes values only in $[0..\lceil 4E \rceil]$ such that the expectation is not lowered. We suspect that this result may be convenient in many other such situations, e.g., when using additive drift in processes that may overshoot the target. 

\begin{lemma}\label{lem:cap}
  Let $Y$ be a random variable taking values in the non-negative integers such that $Y \succeq \Bin(k,p)$ for some $k \in \N$ and $p \in [0,1]$ with $kp \ge 1$. Let $E = kp$ denote the expectation of\/ $\Bin(k,p)$. Then there is a random variable $Z$ such that 
  \begin{itemize}
  \item $\Pr[Z=i] = \Pr[Y=i]$ for all $i \in [0..E]$,
  \item $\Pr[Z=i] = 0$ for all $i \ge 4E+1$,
  \item $E[Z] \ge E$.
  \end{itemize}
\end{lemma}

\begin{proof}
  Let $Z$ be defined by $\Pr[Z=i] = \Pr[Y=i]$ for all $i \in [0..E]$ and $\Pr[Z = \lceil 4E \rceil] = 1 - \Pr[Y \in [0..E]]$. Then it remains to show that $E[Z] \ge E$. If $X \sim \Bin(k,p)$, and hence $E = E[X]$, then $\Pr[X > E] \ge \frac 14$ by~\cite{Doerr18exceedexp}. Since $Y \succeq X$, we have $\Pr[Y > E ] \ge \Pr[X > E] \ge \frac 14$. By definition, $\Pr[Z = \lceil 4E \rceil] = \Pr[Y > E] \ge \frac 14$ and thus $E[Z] = \sum_{i = 0}^{\lceil 4E \rceil} i \Pr[Z = i] \ge \lceil 4E \rceil \Pr[Z = \lceil 4E \rceil] \ge \lceil 4E \rceil \cdot \frac 14 \ge E$.\qed
\end{proof}

We now prove Theorem~\ref{thm:second}.

\begin{proof}
Let first $\delta \le 1$. We first analyze the time spend on all states different from $0$. To this aim, let $\tilde X_{t}$, $t = 0, 1, \dots$, be the subprocess where we are above zero. Formally speaking, $\tilde X$ is the subsequence of $(X_{t})$ consisting of all $X_{t}$ that are greater than $0$. Viewed as a random process, this means that we sample the next state according to the same rules as for the $X$-process; however, if this is zero, then immediately and without counting this as a step we sample the new state from the distribution described in \textbf{(0)} conditional on being positive (which is the same as saying that we resample until we obtain a positive result). With this, the distribution describing one step of the process is a distribution on the positive integers such that $(\tilde X_{t+1} \mid \tilde X_{t}) \succeq \Bin(k,(1+\delta) \tilde X_{t} / k)$. 
We may thus apply Theorem~\ref{thm:reverseDriftBinomial} and obtain that after an expected total number of at most
$$
\tfrac{21.6}{1-\gamma_0} D_0 \ln(2 D_0)+ 3.6 \log_2(n) \lceil 3 / \delta \rceil
$$
steps, the process $\tilde X$ reaches or exceeds $n$.

It remains to analyze how many steps the process $X$ spends on state~$0$. To this end we first show the following claim bounding the probability of falling back to $0$ when at a state $x$. The proof of the claim is essentially an adaptation of an argument regarding unbiased random walks (also knows as the \emph{Gamblers Ruin Problem}), see, for example, \cite[Section~12.2]{Mit-Upf:b:05} for a treatment.

\textbf{Claim}: Let $x$ be such that $0 \leq x \leq D_0$ and let $t_0 \geq 0$. We condition on $X_{t_0}=x$. Then the probability that, in the time from $t_0$ on, the process reaches a state of at least $D_0$ before reaching state~$0$ is  at least $x/ ( 4D_0)$.

The claim is trivially true for $x=0$. Thus, suppose $x > 0$. To ease reading, we regard the process $(Y_t)$ defined by $Y_t = X_{t_0 + t}$ for all $t \geq 0$. 
Clearly, $E[Y_0] = x$.

Let $R$ be the first time that $Y$ reaches or exceeds $D_0$, or hits $0$; this is a stopping time. To ease the following argument, we regard the following process $Z$, which equals $Y$ until the stopping time (and hence has the same stopping time). We define $Z$ recursively. We start by setting $Z_0 := Y_0$. Assume that $Z_t$ is defined and $Z_t \cdot \textbf{1}_{Z_t \le D_0} = Y_t \cdot \textbf{1}_{Y_t \le D_0}$. If $Z_t > D_0$, then we set $Z_{t+1} = Z_t$. Otherwise, that is, when $Z_t = Y_t = x \le D_0$ for some $x$, then we recall that $Y_{t+1} \succeq \Bin(k,(1+\delta) x/k) \succeq \Bin(k,x/k)$. In this case, we let $Z_{t+1}$ be the random variable constructed in Lemma~\ref{lem:cap} (w.r.t.~$Y_{t+1}$, $k$, and $p = x/k$). By this lemma, we have $Z_{t+1} \cdot \textbf{1}_{Z_{t+1} \le D_0} = Y_{t+1} \cdot \textbf{1}_{Y_{t+1} \le D_0}$, allowing us to continue our recursive definition of $Z$, and $E[Z_{t+1} \mid Z_t] \ge Z_t$, showing that $(Z_t)$ is a submartingale.
We can thus use the optional stopping theorem 
 to see that $E[Z_R] \geq E[Z_0]$. Furthermore, 
\begin{align*}
E[Z_R] &= \Pr[Z_R \geq D_0] E[Z_R \mid Z_R \geq D_0] + \Pr[Z_R = 0] E[Z_R \mid Z_R = 0] \\
& = \Pr[Z_R \geq D_0] E[Z_R \mid Z_R \geq D_0] \le \Pr[Z_R \geq D_0] \cdot 4D_0,
\end{align*}
the latter again due to Lemma~\ref{lem:cap}.
Consequently
$$
 \Pr[Y_R \geq D_0] = \Pr[Z_R \geq D_0] \geq \frac{E[Z_0]}{4D_0} = \frac{x}{4D_0}.
$$
This shows the claim.

Let $t\geq 0$, let us again condition on $X_t = 0$, and let $A$ be the event that the process reaches a state of at least $D_0$ after time $t$ before reaching a state of $0$. Using the claim and the law of total probability we now see that
\begin{align*}
P[A] & = \sum_{x=0}^\infty P[A \mid X_{t+1}=x]P[X_{t+1}=x]\\
     & \geq \sum_{x=0}^\infty \frac{x}{4D_0}P[X_{t+1}=x]\\
     & = \frac{E[X_{t+1}]}{4D_0} \geq \frac{E_0}{4D_0}.
\end{align*}
We conclude that the number of iterations spent on state~$0$ before reaching a state of at least $D_0$ is dominated by a geometric distribution with success rate $\frac{E_0}{4D_0}$. Consequently, the expected number of these iterations is at most $4D_0/E_0$. 

Once the process has reached a state of $D_0$ or higher, by Theorem~\ref{thm:reverseDriftBinomial} the probability to ever return to $0$ is at most $0.7218$. Hence the expected number of times this happens is at most $1/0.2782$. We can now use Wald's equation (Theorem~\ref{thm:wald}) to obtain the desired run time result.

The case of $\delta > 1$ is  analogous with $32$ instead of $D_0$ and using Lemma~\ref{lem:deltalarge} instead of Theorem~\ref{thm:reverseDriftBinomial}. \qed
%
%
%
%
\end{proof}


\newcommand{\xmin}{x_{\mathrm{min}}}

\subsection{Processes That Start High}

In condition \textbf{(0)} of the second up-drift theorem (Theorem~\ref{thm:second}), we only exploit the progress made to states not exceeding $D_0$ when leaving state~$0$. When a process has a decent chance to leave $0$ to a state equal to or above $D_0$, then we can ignore the costly first part of the analysis. This is what we analyze in this section by replacing the condition \textbf{(0)} with a start condition~\textbf{(S)} which intuitively says that, at any time of the process (even when not at state~$0$), we have a good chance of starting the process fresh from a rather high minimum value. The proof is an easy combination of Lemma~\ref{lem:climbUp} and a restart argument. To ease the notation, we use the shorthand $\log^0_b(x) := \max\{0, \log_b(x)\}$ for all $x \in \R$ and $b > 1$.
%

\begin{theorem}[Third Multiplicative Up-Drift Theorem]\label{thm:third}
Let $(X_{t})_{t \in \N}$ be a stochastic process over  $\Z_{\geq 0}$. Let $n,k \in \Z_{\geq 1}$,  and $\delta > 0$ such that $n - 1  \le (1+\delta)^{-1} k$. Let $D_0 = \min \{100/\delta,n\}$ when $\delta \le 1$ and $D_0 = \min \{32,n\}$ otherwise. Let $\xmin \geq D_0 > 0$. Assume that for all $t \ge 0$ and all $x \in [0..n-1]$ with $\Pr[X_{t} = x] > 0$, the following two properties hold.
\begin{description}
	\item[(Bin)] If $x \ge \xmin$, then $(X_{t+1} \mid X_{t} = x) \succeq \mathrm{Bin}(k,(1+\delta) x/k)$.
	\item[(S)] $\Pr[ X_{t+1} \geq \xmin \mid X_{t} = x] \geq p$. Also, $\Pr[X_0 \ge \xmin] \ge p$.
\end{description}
Let $T := \min\{t \geq 0 \mid X_{t} \geq n\}$. Then, if $\delta \leq 1$,
\begin{align*}
E[T] 
 & \leq  3.6\left( 1/p + \lceil\log^0_2(n / \xmin)\rceil \lceil 3 / \delta \rceil\right).
\end{align*}

If $\delta > 1$, then we have 
\begin{align*}
E[T]  & \leq 1.3/p + 2.6 \lceil \log^0_{1+\delta}(n / \xmin) \rceil.
\end{align*}
\end{theorem}


\begin{proof}
We start by considering the case $\delta \le 1$. Regardless of where the process is at some time $t_0$, by the start condition~\textbf{(S)} it takes an expected number of at most $1/p$ iterations to again reach at state of at least $\xmin$. Then, by Lemma~\ref{lem:climbUp} and $\xmin \geq D_0$, we see that the time to reach or exceed $n$ when starting at $\xmin$ or higher is no more than another $\lceil\log^0_2(n / \xmin)\rceil$ iterations, with a probability of at least $0.2782$.

In case this fails (with probability at most $1 - 0.2782$), we simply restart the argument at the current state. By Wald's equation (Theorem~\ref{thm:wald}), the expected time to reach or exceed $n$ is at most
\[
\frac{1}{0.2782}\left(1/p + \lceil\log^0_2(n / \xmin)\rceil \lceil 3 / \delta \rceil\right).
\]
The claim follows from noting that $1/0.2782 < 3.6$. 

For $\delta > 1$, we proceed similarly. It again takes an expected number of $1/p$ iterations to reach $\xmin$ or higher. If $\xmin \ge n$, we are done. Otherwise, we invoke Lemma~\ref{lem:deltalarge} to see that with probability at least $0.78$, the process increases by a factor of at least $(1+\delta/2)$ in each subsequent iteration (that starts below $n$), using $\xmin \geq D_0 \geq 32$. In this case, using again Equation~\eqref{eq:deltaHalfInLogBase}, we reach $n$ in at most $\lceil \log_{1+\delta/2}(n/\xmin) \rceil \le 2 \lceil \log_{1+\delta}(n / \xmin) \rceil$ iterations. With a restart argument used in the failure case (occurring with probability $0.22$), we obtain the claimed expected hitting time of $\tfrac 1 {0.78} (1/p + 2 \lceil \log_{1+\delta}(n / \xmin) \rceil) \le 1.3/p + 2.6 \lceil \log_{1+\delta}(n / \xmin) \rceil$ by using Wald's equation (Theorem~\ref{thm:wald}).\qed
\end{proof}

\section{The Level-Based Theorem}

In this section, we apply our up-drift theorems to give an insightful proof of a sharper version of the level-based theorem first proposed by Lehre~\cite{Leh:c:11:LevelBased}. 

The general setup of such level-based theorems is as follows. There is a ground set $\X$, which in typical applications is the search space of an optimization problem. On this ground set, a Markov process $(P_t)$ induced by a population-based EA is defined. We consider populations of fixed size $\lambda$, which may contain elements several times (multi-sets). We write $\X^\lambda$ to denote the set of all such populations. We only consider Markov processes where each element of the next population is sampled independently with repetition. That is, for each population $P \in \X^\lambda$, there is a distribution $D(P)$ on $\X$ such that given $P_t$, the next population $P_{t+1}$ consists of $\lambda$ elements of $\X$, each chosen independently according to the distribution $D(P_t)$. As all our results hold for any initial population $P_0$, we do not make any assumptions on $P_0$.

In the level-based setting, we assume that there is a partition of $\X$ into levels $A_1, \dots, A_m$. Based on information in particular on how individuals in higher levels are generated, we aim for an upper bound on the first time such that the population contains an element of the highest level $A_m$. The first such result was given in~\cite{Leh:c:11:LevelBased}. Improved and easier to use versions can be found in~\cite{DangL16,CorusDEL18plus}. 

To ease the comparison with our result, we now state the strongest level-based theorem before our work. We note that (i)~the time bound has a quadratic dependence on $\delta$ and (ii)~the population size needs to be $\Omega(\delta^{-2} \log(\delta^{-2}))$.

\begin{theorem}[\cite{CorusDEL18plus}]\label{thm:OldLevelBasedTheorem}
Consider a population process as described above. Let $(A_1,\ldots,A_m)$ be a partition of $\X$. We write $A_{\ge j} := \bigcup_{i=j}^m A_i$ for all $j \in [1..m]$. 
Assume that there are $z_1,\ldots,z_{m-1},\delta \in (0,1]$ and $\gamma_0 \in (0,1)$ such that, for any population $P \in \X^\lambda$, the following three conditions are satisfied.
\begin{description}
	\item[(G1)] For each level $j \in [1..m-1]$, if $|P \cap A_{\geq j}| \geq \gamma_0 \lambda$, then
	$$
	\Pr_{y \sim D(P)} [y \in A_{\geq j+1}] \geq z_j.
	$$
	\item[(G2)] For each level $j \in [1..m-2]$ and all $\gamma \in (0,\gamma_0]$, if $|P \cap A_{\geq j}| \geq \gamma_0 \lambda$ and $|P \cap A_{\geq j+1}| \geq \gamma \lambda$, then
	$$
	\Pr_{y \sim D(P)}[y \in A_{\geq j+1}] \geq (1+\delta)\gamma.
	$$ 
	\item[(G3)] The population size $\lambda$ satisfies
	$$
	\lambda \geq \frac{4}{\gamma_0\delta^2} \ln \left( \frac{128m}{z^*\delta^2} \right)\mbox{, where }z^* = \min_{j \in [1..m-1]} z_j.
	$$
\end{description}
Let $T := \min \set{\lambda t}{P_t \cap A_m \neq \emptyset}$.
Then we have
$$
E[T] \leq 8 \frac{\lambda}{\delta^2} \sum_{j=1}^{m-1} \left( \ln \left( \frac{6\delta\lambda}{4 + z_j\delta\lambda} \right) + \frac{1}{\lambda  z_j} \right).
$$
\end{theorem}
The proof given in~\cite{CorusDEL18plus}, as the previous proofs of level-based theorems, uses drift theory with an intricate potential function.


We now derive from our multiplicative up-drift theorems a version of the level-based theorem with (tight) linear dependence on $\delta$. This theorem is further improved with respect to the version given in~\cite{DoerrK19} by only requiring a population size that depends linearly on $\delta$ (rather than an at least quadratic dependence as in~\cite{DoerrK19} or in the previous-best version given in Theorem~\ref{thm:OldLevelBasedTheorem}). To allow such much smaller population sizes to suffice, we need a slightly stronger assumption on making improvements (as can be seen in \textbf{(G1)} and \textbf{(G2)} compared between Theorems~\ref{thm:OldLevelBasedTheorem} and~\ref{thm:NewLevelBasedTheorem}, where an additional factor of $1/4$ is inserted). We do not see any realistic situations in which the assumptions of Theorem~\ref{thm:OldLevelBasedTheorem} are fulfilled, but ours are not.

For the (technically more demanding) case $\delta \le 1$, we show the following result. We treat the easier case $\delta > 1$, not discussed in any previous work, separately at the end of this section.

\begin{theorem}[Level-Based Theorem]\label{thm:NewLevelBasedTheorem}
Consider a population-based process as described in the beginning of this section.

Let $(A_1,\ldots,A_m)$ be a partition of $\X$. Let $A_{\ge j} := \bigcup_{i=j}^m A_i$ for all $j \in [1..m]$. Let $z_1,\ldots,z_{m-1},\delta \in (0,1]$, and let $\gamma_0 \in (0,\frac{1}{1+\delta}]$  with $\gamma_0 \lambda \in \Z$. Let $D_0 = \min\{\lceil 100/\delta \rceil,\gamma_0 \lambda\}$ and $c_1 = 80 \, 000$. Let
$$
t_0 = \frac{10^4}{\delta} \left(m + \frac{1}{1-\gamma_0} \sum_{j=1}^{m-1} \log^0_2\left(\frac{2\gamma_0\lambda}{1+\frac{z_j \lambda}{D_0}}\right) + \frac{1}{\lambda} \sum_{j=1}^{m-1}\frac{1}{z_j} \right),
$$
where $\log^0_2(x) := \max\{0,\log_2(x)\}$ for all $x \in \R$. Assume that for any population $P \in \X^\lambda$ the following three conditions are satisfied.
\begin{description}
	\item[(G1)] For each level $j \in [1..m-1]$, if $|P \cap A_{\geq j}| \geq \gamma_0 \lambda {/4}$, then
	$$
	\Pr_{y \sim D(P)} [y \in A_{\geq j+1}] \geq z_j.
	$$
	\item[(G2)] For each level $j \in [1..m-2]$ and all $\gamma \in (0,\gamma_0]$, if $|P \cap A_{\geq j}| \geq \gamma_0 \lambda {/4}$ and $|P \cap A_{\geq j+1}| \geq \gamma \lambda$, then
	$$
	\Pr_{y \sim D(P)} [y \in A_{\geq j+1}] \geq (1+\delta)\gamma.
	$$ 
	\item[(G3)] The population size $\lambda$ satisfies
	$$
	\lambda \geq {\frac{338}{\gamma_0 \delta} \ln \left(8 t_0 \right)}.
	$$
\end{description}

Then $T := \min \set{\lambda t}{P_t \cap A_m \neq \emptyset}$ satisfies
\begin{align*}
E[T] 
 & \leq 8\lambda t_0 = c_1 \frac{\lambda}{\delta} \left(m + \frac{1}{1-\gamma_0} \sum_{j=1}^{m-1} \log^0_2\left(\frac{2\gamma_0\lambda}{1+\frac{z_j \lambda}{D_0}}\right) + \frac 1 \lambda \sum_{j=1}^{m-1}\frac{1}{z_j} \right).
\end{align*}
\end{theorem}

Note that, with $z^* = \min_{j \in [1..m-1]} z_j$ and $\gamma_0$ a constant, \textbf{(G3)} in the previous theorem is satisfied for some $\lambda$ with
$$
\lambda = O\left(\frac{1}{\delta}\log\left( \frac{m}{\delta z^*} \right)\right)
$$
as well as for all larger $\lambda$.

We now compare our new level-based theorem with the previous best result (Theorem~\ref{thm:OldLevelBasedTheorem}). Since we do not try to optimize constant factors, we do not discuss these (but note that ours are large). 

We first observe that as long as $\gamma_0$ can be assumed to be a constant bounded away from $1$, then our bound for any values of the variables is at most a constant factor larger than the bound of Theorem~\ref{thm:OldLevelBasedTheorem}. When $z_j \lambda$ is large, the $\log^0_2(\cdot)$ expression can degenerate to an expression of order $\log(D_0) = O(\log(1/\delta))$. This cannot happen for the logarithmic expression in the run time bound of Theorem~\ref{thm:OldLevelBasedTheorem}, however, even in this case, our bound is of order $O(\log(1/\delta)/\delta)$, whereas the previous best result was $O(\delta^{-2})$. Hence when ignoring constant factors and assuming $\gamma_0<1$ a constant, our bound is at least as strong as the previous results.

In terms of asymptotic differences, we first note the improved dependence of the run time guarantee on $\delta$. Ignoring a possible influence of $\delta$ on the logarithmic terms in the run time estimate, the dependence now is only $O(\delta^{-1})$, whereas it was $O(\delta^{-2})$ in the previous result.

The second asymptotic difference concerns the minimum value for $\lambda$ that is prescribed by condition \textbf{(G3)}. Note that in both results the run time estimate is a sum of two terms, the first depending linearly on $\lambda$. Consequently, being able to use a smaller population size $\lambda$ can improve the run time. The main difference, and again ignoring the logarithmic term in \textbf{(G3)}, is that $\lambda$ has to be $\Omega(\delta^{-2})$ in the previous result and only $\Omega(\delta^{-1})$ in ours. The logarithmic terms are more tedious to compare, but clearly ours is asymptotically not larger as long as $\lambda$ is at most exponential in $m$ or at most exponential in~$1/z^*$.

We continue by discussing minor differences between the two results. We note that $t_0$ in our result depends on $\lambda$. We thus end up in the slightly annoying situation that in our version, $\lambda$ appears also in the right-hand side of \textbf{(G3)}. However, since $\lambda$ appears on the right-hand side only inside a logarithm (and one that is at least $\ln(m)$), it is usually not difficult to find solutions for this inequality that lead to an asymptotically optimal value~$\lambda$.

One key difference is that both \textbf{(G1)} and \textbf{(G2)} impose a condition from the point on when at least $\gamma_0 \lambda /4$ individuals are on a level, whereas the previous level-based theorem (as the conference version of this work) only does so from $\gamma_0 \lambda$ on. This additional slack is required to bring down the dependence of $\lambda$ on $1 / \delta$ from essentially quadratic to essentially linear. We do not see any realistic application where the stronger versions of \textbf{(G1)} and \textbf{(G2)} would be harder to show than the previous ones.

In summary, when ignoring constant factors, we do not see any noteworthy downsides of our new result and we did not find any result previously proven via a level-based theorem that could not be proven with our result. At the same time, the superior asymptotics of the run time bound and the minimum requirement on $\lambda$ in terms of $\delta$ clearly are an advantage of our result.


We now proceed with proving the new level-based theorem. We shall use an estimate for the probability that a binomial random variable is a least its expectation. The following result was proven with elementary means in~\cite{Doerr18exceedexp}. A very similar result was shown with deeper methods in~\cite{GreenbergM14}.

\begin{lemma}\label{lem14}
  Let $n \in \N$ and $p \ge \frac 1n$. Let $X \sim \Bin(n,p)$. Then \[\Pr[X \ge E[X]] \ge \frac 14.\]
\end{lemma}

We are now ready to state the formal proof of Theorem~\ref{thm:NewLevelBasedTheorem}.

\begin{proof}
We first note that $t_0 \geq 10^4$, so from \textbf{(G3)}, we have
\begin{equation}
\gamma_0 \lambda \ge 200/\delta \geq 200. \label{eq:gl}
\end{equation}

We say that we lose level $j$ if, before having optimized, there is a time $t$ at which there are at least $\gamma_0 \lambda$ individuals at least on level $j$, and a later time $t' > t$ such that at that time there are less than $\gamma_0 \lambda / 4$ individuals at least on level $j$.

Our proof proceeds now as follows. First we will condition on never losing a level. We show that we have multiplicative up-drift for the number of individuals on the lowest level which does not have at least $\gamma_0\lambda$ individuals and a simple induction allows us to go up level by level. Then we show that any level which has at least $\gamma_0\lambda$ individuals will not be lost until the optimization ends, with sufficiently high probability. 

Since we are only interested in the time until we have the first individual in $A_m$, we may assume that condition \textbf{(G2)} also holds for $j=m-1$. 

We now analyze how the number of individuals above the highest level with at least $\gamma_0 \lambda$ individuals develops. Let a level $j \leq m-1$ be given such that $|P \cap A_{\geq j}| \geq \gamma_0 \lambda$. We condition on never losing level $j$, that is, on never having less than $\gamma_0 \lambda/4$ individuals on level $j$ or higher. We let $(X_{t})$ be the random process describing the number of individuals on level $j+1$ or higher, that is, we have $X_{t} = |P_t \cap A_{\geq j+1}|$ for all $t$.

We now distinguish two cases. Suppose first that $z_j \lambda \geq D_0$; this means that we expect at least $D_0$ individuals on the new level in any given iteration. By Lemma~\ref{lem14}, we can apply Theorem~\ref{thm:third} with $p= \frac 14$, $n=\gamma_0 \lambda$, and $\xmin = z_j \lambda$ to see that the level is filled to at least $\gamma_0\lambda$ individuals in an expected time of at most
\begin{align*}
T_j 
&:= 3.6\left( 4 + \lceil\log^0_2(\gamma_0 / z_j)\rceil \lceil 3 / \delta \rceil\right)\\
&\le 14.4 + 14.4 \, \frac{\lceil \log^0_2(\gamma_0 / z_j) \rceil}{\delta}.
\end{align*}
iterations.

In the second case we have $z_j \lambda < D_0$ and we want to use Theorem~\ref{thm:reverseDriftWithZeroBinomial}, where our target is again to have $n = \gamma_0 \lambda$ individuals on level $j+1$ or higher. We start by determining a useful $E_0$ for which we can show Condition~\textbf{(0)}. From \textbf{(G1)} we have that if $X_{t} =0$, then the number $Y := X_{t+1}$ of individuals sampled in $A_{\ge j+1}$ follows a binomial law with parameters $\lambda$ and success probability $p \ge z_j$. 


We now estimate $E_0^{(j)} := E[\min\{D_0,Y\}]$. Assume first that $\lambda z_j \ge 1$ and hence $E[Y] \ge 1$. By Lemma~\ref{lem14}, we have $E_0^{(j)} \ge \frac {1}{4} \min\{D_0,E[Y]\} = \frac 14 \min\{D_0, \lambda z_j\}$. If instead we have $\lambda z_j < 1$, then the probability to sample at least one individual on a higher level is at least 
$\Pr[Y \ge 1] \ge 1 - (1-z_j)^{\lambda} \ge 1 - \exp(-z_j \lambda) \ge 1 - (1 - \frac 12 z_j \lambda) = \frac 12 z_j \lambda$, using the elementary estimates $1+x \le \exp(x)$ valid for all $x \in \R$ and $1 - \frac 12 x \ge \exp(-x)$ valid for all $0 \le x \le 1$. 
Consequently, in either case, $E_0^{(j)} \ge \frac{1}{4} \min\{D_0, \lambda z_j\}$. Since we will later need to bound the inverse of $E_0^{(j)}$ from above, we note that
\begin{equation}
\frac{1}{E_0^{(j)}} \leq \frac{\delta}{25} + \frac{4}{\gamma_0\lambda} + \frac{4}{\lambda z_j} \leq 2 + \frac{4}{\lambda z_j} 
\end{equation}
by $\delta \leq 1$ and Equation~\eqref{eq:gl}.

From \textbf{(G2)} we see that when $X_{t} > 0$, then the number $X_{t+1}$ of individuals sampled on level $j+1$ or higher stochastically dominates a binomial law with parameters $\lambda$ and $(1+\delta)X_{t} / \lambda$. Consequently, we can apply Theorem~\ref{thm:reverseDriftWithZeroBinomial} and estimate that the expected  number of generations until there are at least $\gamma_0 \lambda$ individuals on level $j+1$ or higher is at most
\begin{align*}
 T'_j &:= \frac{ 4D_0 }{0.2782 E_0^{(j)}}  + \tfrac{21.6}{1-\gamma_0} D_0 \ln(2 D_0)+ 3.6 \log_2(\gamma_0 \lambda) \lceil 3 / \delta \rceil.
 \end{align*}
Since $D_0 \ge \min\{100/\delta,\gamma_0 \lambda\} \geq 100$ by~\eqref{eq:gl} and thus $\ln(2) \leq \ln(D_0) \cdot 0.151$, we have $21.6 \ln(2D_0) = 21.6 (\ln (2) + \ln(D_0)) \leq 25 \ln(D_0)$. With this and $c_0 := 25$ we estimate
\begin{align*}
 T_j' & \leq  c_0 \left( D_0 / E_0^{(j)} + \frac{1}{1-\gamma_0} D_0 \ln(D_0)+ \frac{\log_2(\gamma_0 \lambda)}{\delta} \right)\\
 & \le c_0 \left(D_0\left(2 + \frac{4}{\lambda z_j}\right) + \frac{1}{1-\gamma_0} D_0 \ln(D_0) + \frac{\log_2(\gamma_0\lambda)}{\delta}\right) \\
   & =  c_0 \left(D_0\left(2 + \frac{\ln(D_0)}{1-\gamma_0} \right) + \frac{\log_2(\gamma_0\lambda)}{\delta} + D_0\frac{4}{\lambda z_j}\right)\\
   & \leq  c_0 \left(D_0\left(2\frac{\ln(D_0)}{1-\gamma_0} \right) + \frac{\log_2(\gamma_0\lambda)}{\delta} + \frac{400}{\lambda \delta z_j}\right)\\
  & \leq  c_0 \left( \frac{2 \lceil 100/\delta \rceil\ln(\gamma_0 \lambda)}{1-\gamma_0} + \frac{\log_2(\gamma_0\lambda)}{\delta} + \frac{400}{\lambda \delta z_j} \right)\\
   & \leq  \frac{c_0}{\delta} \left( \frac{203 \log_2(\gamma_0\lambda)}{1-\gamma_0} + \frac{400}{\lambda z_j} \right).
 \end{align*}

Let \[T_j^* = \frac{c_0}{\delta} \left(1 + \frac{203}{1-\gamma_0} \log^0_2\left(\frac{2\gamma_0\lambda}{1+\frac{z_j \lambda}{D_0}}\right) + \frac{400}{\lambda z_j} \right)\]
and note that $T_j^* \ge T_j$ when $z_j \lambda \ge D_0$ and $T_j^* \ge T_j'$ otherwise. Hence $T_j^*$ is an upper bound for the expected time to have at least $\gamma_0 \lambda$ individuals in $A_{\ge j+1}$ when starting with at least $\gamma_0 \lambda$ individuals in $A_{\ge j}$ and assuming that we do not lose level $j$. 

Summing over all levels, we obtain the following bound on the number of steps to reach a search point in $A_m$, still conditional on never losing a level:
\begin{align*}
  \sum_{j=1}^{m-1} T_j^* &\le \frac{400 c_0}{\delta} \left(m + \frac{1}{1-\gamma_0} \sum_{j=1}^{m-1} \log^0_2\left(\frac{2\gamma_0\lambda}{1+\frac{z_j \lambda}{D_0}}\right) + \sum_{j=1}^{m-1}\frac{1}{\lambda z_j} \right)
 = t_0.
\end{align*}

We now argue that, with sufficiently high probability, we indeed do not lose a level. Specifically, we show that, from any iteration with at least $\gamma_0 \lambda/2$ individuals until the next iteration with at least that many individuals, the probability is at most 
$$
\exp\left( - \frac{\delta \gamma_0 \lambda/2}{169} \right)
$$
that we have an iteration with less than $\gamma_0 \lambda / 4$ individuals in between (which we will call a failure).

We distinguish two cases: we either have at least $\gamma_0 \lambda$ individuals on the level and above, or less. Using a standard Chernoff bound argument on \textbf{(G2)} with $\gamma=\gamma_0$ we see that, for iterations with at least $\gamma_0 \lambda$ individuals, the probability to fall below $\gamma_0 \lambda/2$ individuals in the next step is at most
\begin{equation*}
\exp(-\gamma_0 \lambda / 8) < \exp\left( - \frac{\delta \gamma_0 \lambda/2}{169}\right).
\end{equation*}
This shows that steps with at least $\gamma_0 \lambda$ individuals lead to a failure with at most the desired small probability.


In the case of less than $\gamma_0 \lambda$ individuals, just as in the proof of Theorem~\ref{thm:first}, we want to apply Lemma~\ref{lem:subm}. In the language of Lemma~\ref{lem:subm}, we have $n = \gamma_0\lambda \geq 200/\delta$ using Equation~\eqref{eq:gl}. Thus, we can use Lemma~\ref{lem:subm} to estimate the probability of falling below $\gamma_0\lambda/4$ after having reached at least $D \geq \gamma_0 \lambda/2 \geq 100/\delta$ individuals. We thus see that this failure probability is at most
$$
\exp\left( - \frac{\delta D}{169} \right) 
   \leq \exp\left( - \frac{\delta \gamma_0 \lambda/2}{169} \right).
$$
Thus, also in this case the probability of failure is small. Using \textbf{(G3)}, we see that the last term is at most $1/(8t_0)$. In order to obtain the overall failure probability over any number of $t$ steps, we can now make a union bound over all intervals, each ranging from one iteration with at least $\gamma_0 \lambda/2$ individuals to the next. For this we will pessimistically assume that we have $t$ such intervals within $t$ steps.
Thus, we see that the probability of ever losing a level within $2t_0$ steps (twice the conditional expected optimization time, conditional on not losing a level) is at most $p_1 :=0.25$. Using Markov's inequality, the probability of successful optimization within $2t_0$ iterations without losing a level is at least $p_2 := 0.5$. Thus, with a union bound on the failure probabilities, we get an unconditional probability of successful optimization within $2t_0$ iterations of at least $1 - p_1 - p_2 = 0.25$. Thus, a simple restart argument shows that the expected time (in iterations) for optimization is at most $8t_0$, giving the desired run time bound.
\ignore{
Consider now a level $j \leq m-1$ such that $|P \cap A_{\geq j}| \geq \gamma_0 \lambda$. We use \textbf{(G2)} to see that the probability of any generated individual to be at least on level $j$ is 
$$
\Pr_{y \sim D(P)} [y \in A_{\geq j}] \geq (1+\delta)\gamma_0.
$$
Thus, the expected number of generated individuals on level $j$ is at least $(1+\delta)\gamma_0 \lambda$. We now want to determine the probability of undershooting this expected value by a factor of $1-\delta/2$; for this we use a multiplicative Chernoff bound and see that this probability is at most
$$
\exp(-\delta^2 \gamma_0 \lambda / 8) \stackrel{\textbf{(G3)}}{\leq} \left( k m \frac{\log (\lambda) + (z^*)^{-1}/\lambda}{\delta} \right)^{-1}.
$$
Then, after $k m \frac{\log (\lambda) + (z^*)^{-1}/\lambda}{\delta} /2$ generations, the process is done with probability at least $1/2$, if no level was ever lost. We bound the probability of ever losing a level in this time by a union bound with $1/2$. Since, conditional on never losing a level, we succeed in this time with probability at least $1/2$, we succeed overall with probability at least $1/4$. A simple restart argument concludes the proof.
}%
\qed
\end{proof}

We now discuss the case $\delta > 1$. With similar, often easier arguments, we prove the following result.

\begin{theorem}[Level-Based Theorem for $\delta > 1$]\label{thm:NewLevelBasedTheoremDeltaLarge}
Consider a population-based process as described in the beginning of this section.

Let $(A_1,\ldots,A_m)$ be a partition of $\X$. Let $A_{\ge j} := \bigcup_{i=j}^m A_i$ for all $j \in [1..m]$. Let $z_1,\ldots,z_{m-1} \in (0,1]$, $\delta > 1$, and $\gamma_0 \in (0,\frac{1}{1+\delta}]$ with $\gamma_0 \lambda \in \Z_{\ge 32}$. Let 
\[
t_0 = 101.6 m + 2.6 \sum_{j=1}^{m-1} \log^0_{1+\delta}\left(\frac{2\gamma_0\lambda}{1 + \frac{z_j \lambda}{D_0}} \right) + \frac{657}{\lambda} \sum_{j=1}^{m-1} \frac 1 {z_j}.
\]
Assume that for any population $P \in \X^\lambda$ the following three conditions are satisfied.
\begin{description}
	\item[(G1)] For each level $j \in [1..m-1]$, if $|P \cap A_{\geq j}| \geq \gamma_0 \lambda$, then
	$$
	\Pr_{y \sim D(P)} [y \in A_{\geq j+1}] \geq z_j.
	$$
	\item[(G2)] For each level $j \in [1..m-2]$ and all $\gamma \in (0,\gamma_0]$, if $|P \cap A_{\geq j}| \geq \gamma_0 \lambda$ and $|P \cap A_{\geq j+1}| \geq \gamma \lambda$, then
	$$
	\Pr_{y \sim D(P)} [y \in A_{\geq j+1}] \geq (1+\delta)\gamma.
	$$ 
	\item[(G3)] The population size $\lambda$ satisfies $\lambda \ge \frac{4}{\gamma_0} \ln(9t_0)$.
\end{description}

Then $T := \min \set{\lambda t}{P_t \cap A_m \neq \emptyset}$ satisfies
\begin{align*}
E[T] & \leq 9 \lambda t_0 \le 915 \lambda m + 24 \lambda \sum_{j=1}^{m-1} \log^0_{1+\delta}\left(\frac{2\gamma_0\lambda}{1 + \frac{z_j \lambda}{D_0}} \right) + 6000 \sum_{j=1}^{m-1} \frac 1 {z_j}.
\end{align*}
\end{theorem}

The assumption that $\gamma_0 \lambda \ge 32$ is not strictly necessary, but eases the presentation. Note that \textbf{(G3)} and $t_0 \ge 101.6$ already imply $\gamma_0 \lambda \ge 27.27$. Conditions \textbf{(G1)} and \textbf{(G2)} are identical with the case $\delta \le 1$ except that we only require them to hold for $|P \cap A_{\ge j}| \ge \gamma_0 \lambda$ instead of $|P \cap A_{\ge j}| \ge \gamma_0 \lambda/4$. Condition \textbf{(G3)} is of a similar type as in the case $\delta \le 1$.

\begin{proof}
The proof reuses many arguments from the proof for the case $\delta \le 1$. To later apply the second multiplicative up-drift theorem, let $D_0 = \min\{32,\gamma_0 \lambda\}$ and note that by our assumption $D_0 = 32$. 

Mildly different from the case $\delta \le 1$, we now say that we lose a level in iteration $t$ if there is a $j \in [1..m-1]$ such that $|P_t \cap A_{\ge j}| \ge \gamma_0 \lambda$ and $|P_{t+1} \cap A_{\ge j}| < \gamma_0 \lambda$. 

We again condition on never losing a level and later revoke this assumption with a restart argument. Let $j \in [1..m-1]$ and assume that at some time $t'$ we have $|P_{t'} \cap A_{\ge j}| \ge \gamma_0 \lambda$. We analyze how the number of individuals on levels above $j$ develops. To this aim, let $X_t = |P_{t'+t} \cap A_{\ge j+1}|$ for all $t = 0, 1, 2, \dots$. As in the analysis of the case $\delta \le 1$, we distinguish two cases. When $z_j \lambda \ge D_0$, then we can again apply Theorem~\ref{thm:third} with $p = 1/4$, $\xmin = z_j \lambda$, and $n = \gamma_0 \lambda$, showing that the expected time to fill level $j+1$ to at least $\gamma_0 \lambda$ elements is at most 
\[1.3/p + 2.6 \lceil \log^0_{1+\delta}(n/\xmin)\rceil \le 7.8 + 2.6 \log^0_{1+\delta}(\gamma_0/z_j).
\]

If instead we have $z_j \lambda < D_0$, we argue as follows. We have $E[\min\{X_{t+1},D_0\} \mid X_t = 0] \ge \frac 14 \min\{D_0, \lambda z_j\} =: E_0^{(j)}$. We estimate 
\[\frac{1}{E_0^{(j)}} \le \frac 4 {D_0} + \frac{4}{\lambda z_j} = \frac 18 + \frac{4}{\lambda z_j}.\]
With \textbf{(G2)}, we again invoke Theorem~\ref{thm:second} and obtain that the expected number of iterations to have $X_t \ge \gamma_0 \lambda$ is at most 
\[\frac{128}{0.78 E_0^{(j)}} + 2.6 \log_{1+\delta}(\gamma_0 \lambda) + 81 \le 101.6 + \frac{657}{\lambda z_j} + 2.6 \log_{1+\delta}(\gamma_0 \lambda).\]

In either case, $z_j \lambda \ge D_0$ or $\gamma_0 \lambda < D_0$, this level filling-up time is at most 
\[
101.6 + \frac{657}{\lambda z_j} + 2.6 \log^0_{1+\delta}\left(\frac{2\gamma_0\lambda}{1 + \frac{z_j \lambda}{D_0}} \right)\]
in expectation. Summing over all levels, we see that the expected time to, one after the other, fill all levels is at most $t_0$ when we condition on never losing a level. 

The probability to lose the current level in one iteration, by a simple Chernoff bound and~\textbf{(G2)}, is at most $\exp(-\frac 14 \gamma_0 \lambda)$, since we expect to have at least $(1+\delta) \gamma_0 \lambda \ge 2 \gamma_0 \lambda$ offspring on this level or higher. By \textbf{(G3)}, this probability is at most $1 / 9t_0$. By a simple union bound, we see that the probability to lose a level in $3t_0$ iterations is at most $1/3$. Under this assumption, the probability to not find a search point in $A_m$ in the first $3t_0$ iterations is at most $1/3$ by Markov's inequality. Hence with probability $1/3$, we find the desired solution in $3 t_0$ iterations. A simple restart argument with an expected number of three restarts now shows $E[T] \le \lambda \cdot 9 t_0$ as claimed.\qed
\end{proof}

\section{Applications}
\label{sec:applications}

With the improved level-based theorem, we easily obtain the following three results. The first two improve previous results that were obtained via level-based theorems in the case of small $\delta$. The last result shows that our level-based theorem for the case $\delta > 1$ can lead to results better than what was known before for the case $\delta \le 1$ (including using $\delta \le 1$ when $\delta$ actually is larger).

\subsection{Fitness-Proportionate Selection}

Dang and Lehre~\cite{DangL16} show that fitness-proportionate selection can be efficient when the mutation rate is very small; in contrast to previous results that show, for the standard mutation rate $1/n$, that fitness-proportionate selection can lead to exponential run times~\cite{HappJKN08,NeumannOW09}. More precisely, Dang and Lehre regard the $(\lambda,\lambda)$ EA with fitness-proportionate selection for variation and standard bit mutation as variation operator (Algorithm~\ref{alg:fps}). Here fitness-proportionate selection (with respect to a non-negative fitness function $f$) means that from a given population $x_1, \dots, x_\lambda$ we choose a random element such that $x_i$ is chosen with probability $f(x_i) / \sum_{j=1}^\lambda f(x_j)$. When $\sum_{j=1}^\lambda f(x_j)$ is zero, we choose an individual uniformly at random. 

\begin{algorithm2e}%
	Initialize $P_0$ as multi-set of $\lambda$ individuals chosen independently and uniformly at random from  $\{0,1\}^n$\;
  \For{$t=1,2,3,\ldots$}{
    $P_{t} \assign \emptyset$\;
    \For{$i = 1$ \KwTo $\lambda$}{
      select $x \in P_{t-1}$ via fitness-proportional selection\;
      generate $y$ from $x$ by flipping each bit independently with probability~$\pmut$\;
      $P_{t} \assign P_{t} \cup \{y\}$\;
      }
    }
\caption{The $(\lambda,\lambda)$ EA with fitness-proportionate selection and mutation rate $\pmut$ to maximize a function $f : \{0,1\}^n \to \R_{\ge 0}$.}
\label{alg:fps}
\end{algorithm2e}

Dang and Lehre show that this algorithm with mutation rate $\pmut = \frac 1 {6n^2}$ and population size $\lambda = bn^2 \ln n$ for some constant $b>0$ optimizes the \onemax and \leadingones benchmark functions in an expected number of $O(n^8 \log n)$ fitness evaluations. We note that the previous improved level-based theorem (Theorem~\ref{thm:OldLevelBasedTheorem}) would give a bound of $O(n^5 \log^2 n)$ for the smallest-possible choice of $\lambda$. With our tighter version of the level-based theorem, we obtain the following results. 

\begin{theorem}
  Consider the $(\lambda,\lambda)$ EA with fitness-proportionate selection, 
  \begin{itemize}
  \item with population size $\lambda \ge c n \ln(n)$ with $c$ sufficiently large and $\lambda = O(n^K)$ for some constant $K$, and 
  \item mutation rate $\pmut \le \frac 1 {4n^2}$ and $\pmut = \Omega(n^{-k})$ for some constant $k$. 
  \end{itemize}
  Then this algorithm optimizes \onemax in an expected number of $O(\lambda n^2 \log n + n \log(n) / \pmut)$ fitness evaluations, which is $O(n^3 (\log n)^2)$ for optimal parameter choices. It optimizes \leadingones in time $O(\lambda n^2 \log n + n^2 / \pmut)$ fitness evaluations, which becomes $O(n^4)$ with optimal parameter choices.
\end{theorem}

\begin{proof}
  Let $f$ be the function \onemax. We apply Theorem~\ref{thm:NewLevelBasedTheorem} with $\gamma_0 = \frac 12$ and the partition formed by the sets $A_i := \{x \in \{0,1\}^n \mid f(x) = i-1\}$ with $i = 1, 2, \dots, n+1 =: m$.  
  
  To show \textbf{(G1)}, assume that we have at least $\gamma_0 \lambda / 4$ individuals with fitness at least $j$ for some $j \in [0..n-1]$. Since the selection operator favors individuals with higher fitness, the probability that the parent of a particular offspring has fitness at least $j$, is at least $\gamma_0/4$. Assume that such a parent was chosen (and that this does not have fitness $n$ since we would be done then anyway). If the parent has fitness exactly $j$, then the probability to generate a strictly better search point is at least $(n-j) \pmut (1 - \pmut)^{n-1} \ge (n-j) \pmut (1 - (n-1)\pmut) = (n-j) \pmut (1 - o(1))$ by Bernoulli's inequality and $\pmut = o(\frac 1n)$. If the parent has already a fitness of $j+1$ or better, then the probability to generate an offspring of fitness $j+1$ or better is even higher, namely by simply flipping zero bits such an offspring is generated with probability at least $(1-\pmut)^n \ge 1 - n\pmut = 1 - o(1)$. Hence in either case we have \textbf{(G1)} satisfied with $z_j = (n-j) \gamma_0 \pmut (1 - o(1)) / 4$.

  To show \textbf{(G2)}, let $j \in [0..n-2]$, $\gamma \in (0, \gamma_0]$ and $P$ be a population such that at least $\gamma \lambda$ individuals have a fitness of at least $j+1$ and at least $\gamma_0 \lambda/4$ individuals have a fitness of at least~$j$. Let $F^+$ be the sum of the fitness values of the individuals of fitness at least $j+1$ and let $F^- = \sum_{x \in P} f(x) - F^+$ be the sum of the remaining fitness values. By our assumption, $F^+ \ge \gamma \lambda (j+1)$. The probability that an individual of fitness $j+1$ or more is chosen as parent of a particular offspring is 
\begin{align*}
\frac{F^+}{\sum_{x \in P} f(x)} &= \frac{F^+}{F^++F^-} \\
&\ge \frac{\gamma \lambda (j+1)}{\gamma \lambda (j+1)+F^-} \\
&\ge \frac{\gamma \lambda (j+1)}{\gamma \lambda (j+1)+(1-\gamma)\lambda j} \\
&= \gamma \left(1+\frac{1-\gamma}{j+\gamma}\right) \ge \gamma \left(1 + \frac{\frac12}{j+\frac 12}\right) \ge \gamma\left(1 + \frac 1 {2n}\right).
\end{align*}
The probability that a parent creates an identical offspring is $(1 - \pmut)^n \ge 1 - n\pmut$. Consequently, the probability that an offspring has fitness at least $j+1$ is at least $\gamma$ times $(1+\frac 1 {2n}) (1 - n\pmut) \ge 1 + \frac 1 {2n} - n \pmut - O(n^{-2}) \ge 1 + \frac 1 {4n} - O(n^{-2}) =: 1 + \delta$. With this $\delta = \Theta(1/n)$, we have satisfied \textbf{(G2)}.  
  
  Finally, we observe that 
  \begin{align*}
\frac{338}{\gamma_0 \delta} &\ln \left(\frac{c_1}{\delta} \left( \frac{m \log_2(\gamma_0\lambda)}{1-\gamma_0} + \frac{1}{\lambda} \sum_{j=1}^{m-1} \frac{1}{z_j} \right)\right)\\
& = O\left(\frac 1 \delta \log\left(\frac m \delta \left(\log \lambda + \frac 1 {\lambda \pmut}\right)\right)\right)\\
& =  O(n \log n), 
  \end{align*}
  since $m$, $\lambda$, and $1/\pmut$ are polynomially bounded in $n$. This shows \textbf{(G3)}.
  
  Consequently, we can employ Theorem~\ref{thm:NewLevelBasedTheorem} and derive an expected optimization time of 
  \begin{align*}
  E[T] 
  & \leq \lambda \frac{c_1}{\delta} \left( \frac{m \log_2(\gamma_0\lambda)}{1-\gamma_0} + \frac{1}{\lambda} \sum_{i=1}^{m-1} \frac{1}{z_j} \right)\\
  & = O\left( \frac{\lambda m \log \lambda}{\delta}  + \frac{1}{\delta} \sum_{j=1}^{m-1} \frac{1}{(n-j) \pmut}\right)\\
  & = O(\lambda n^2 \log n + n \log(n) / \pmut).
  \end{align*}
which is $O(n^3 \log^2 n)$ for $\lambda = \Theta(n \log n)$ and $\pmut = \Omega(n^{-2} (\log n)^{-1})$.

For $f$ being the \leadingones function, we take the same partition of the search space and also $\gamma_0 = \frac 12$. With similar arguments as above, we show \textbf{(G1)} with $z_j = \gamma_0 \pmut (1-o(1)) / 4$. The proof of \textbf{(G2)} remains valid without changes, since the central argument was that with sufficiently high probability a copy of the parent is generated (hence again we have $\delta = \Theta(1/n)$). The proof of \textbf{(G3)} remains valid since we estimated the $z_j$ uniformly as $z_j = \Omega(\pmut)$. Consequently, we obtain from  Theorem~\ref{thm:NewLevelBasedTheorem} that the optimization time $T$ satisfies
  \begin{align*}
  E[T] 
  & \leq \lambda \frac{c_1}{\delta} \left( \frac{m \log_2(\gamma_0\lambda)}{1-\gamma_0} + \frac{1}{\lambda} \sum_{i=1}^{m-1} \frac{1}{z_j} \right)\\
  & = O\left( \frac{\lambda m \log \lambda}{\delta}  + \frac{m}{\delta \pmut}\right)\\
  & = O(\lambda n^2 \log n + n^2 / \pmut).
  \end{align*}
  This is $O(n^4)$ for $\lambda = O(n^2 / \log n)$ and $\pmut = \Theta(n^{-2})$.\qed
  \end{proof}

\subsection{Partial Evaluation}

Also in Dang and Lehre~\cite{DangL16} a different parent selection algorithm was considered, $2$-tournament selection, where a parent is chosen by picking two individuals uniformly at random and the fitter one is allowed to produce one offspring (see Algorithm~\ref{alg:tts}).

\begin{algorithm2e}%
	Initialize $P_0$ as multi-set of $\lambda$ individuals chosen independently and uniformly at random from  $\{0,1\}^n$\;
  \For{$t=1,2,3,\ldots$}{
    $P_{t} \assign \emptyset$\;
    \For{$i = 1$ \KwTo $\lambda$}{
      select $x_0,x_1 \in P_{t-1}$ uniformly at random\;
			select $x \in \{x_0,x_1\}$ with maximal fitness (breaking ties uniformly)\;
      generate $y$ from $x$ by flipping each bit independently with probability~$\pmut$\;
      $P_{t} \assign P_{t} \cup \{y\}$\;
      }
    }
\caption{The $(\lambda,\lambda)$ EA with $2$-tournament selection and mutation rate $\pmut$ to maximize a function $f : \{0,1\}^n \to \R_{\ge 0}$.}
\label{alg:tts}
\end{algorithm2e}

The test functions they considered were \onemax and \leadingones under partial evaluation (a scheme for randomizing a given function), which we here define only for \onemax. Given a parameter $c \in (0,1)$, we use $n$ i.i.d.\ random variables $(R_i)_{i \leq n}$, each Bernoulli-distributed with parameter $c$. $\onemax_c$ is defined such that, for all bit strings $x \in \{0,1\}^n$, $\onemax_c(x) = \sum_{i=1}^n R_i x_i$. With other words, a bit string has a value equal to the number of $1$s in it, where each $1$ only counts with probability $c$.

Dang and Lehre \cite{DangL16} showed the following statement as part of their core proof \cite[proof of Theorem~21]{DangL16} regarding the performance of Algorithm~\ref{alg:tts} on $\onemax_c(x)$.

\begin{lemma}\label{lem:partialEvaluation}
Let $n$ be large and $c \in (1/n,1)$. Then there is an $a$ such that, for all $\gamma \in (0,1/2)$, the probability to produce an offspring (line 7 of Algorithm~\ref{alg:tts}) of at least the quality of the $\gamma \lambda$-ranked individual of the current population is at least $\gamma (1+a \sqrt{c/n})$.
\end{lemma}

Using their old level-based theorem (with a dependence on $\delta$ of order $5$) and the best possible choice for $\lambda$, they obtain a bound for the expected number of fitness evaluations until optimizing \onemax with partial evaluation with parameter $c \geq 1/n$ of
$$
O\left( \frac{n^{4.5} \log n}{c^{3.5}} \right).
$$
Using the more refined level-based theorem from \cite{CorusDEL18plus}, see Theorem~\ref{thm:OldLevelBasedTheorem} (with a quadratic dependence on $\delta$), one can find a run time bound of
$$
O\left( \frac{n^{3} \log n}{c^{2}} \right).
$$
With our level-based theorem given in Theorem~\ref{thm:NewLevelBasedTheorem} (with a linear dependence on $\delta$), one can prove a run time bound of
$$
O\left( \frac{n^{2} (\log (n))^2}{c} \right).
$$
For this we chose analogously to \cite{DangL16}: $\delta = a \sqrt{c/n}$ as given in Lemma~\ref{lem:partialEvaluation}, $\pmut = \delta/3$, $m=n+1$ (with the partitioning based on fitness), $\gamma_0 = 1/2$, $z_j = 7(1-j/n)(\delta/9)/16$ and $\lambda = b \ln (n)\sqrt{n/c}$ for some constant $b$.


Analogous improvements can be found in the case of \leadingones.

\subsection{Using $\delta > 1$}\label{sec:usingLargeDelta}

In all applications of the level-based theorem in the literature, only the case of $\delta \leq 1$ was used; in fact, the level-based theorem from~\cite{CorusDEL18plus} does not give a version that can benefit from $\delta > 1$ (however, it can always be applied with $\delta=1$ instead of the true $\delta$). We note the following result, which can be improved by taking $\delta > 1$ into account.

Consider optimizing the \leadingones benchmark function using a \mclea with ranking selection and standard bit mutation. When $\lambda \geq 2e \mu$ and $\lambda \geq c \log(n)$ for some specific constant $c$, then an expected run time of $O(n^2 + n \lambda \log(\lambda))$ fitness evaluations is proven in~\cite[Theorem~3(2)]{CorusDEL18plus}. We easily see that in this case, using the partition of the search space into sets of equal fitness, we have $z_j = O(1/n)$ for all $j \in [0..n-1]$ and $\delta = \lambda / e\mu$.

Using our level-based theorem for $\delta > 1$ (Theorem~\ref{thm:NewLevelBasedTheoremDeltaLarge}), we obtain the slightly better bound of $O(n^2 + n\lambda \log_{(1+\lambda/e\mu)}(\lambda))$ since the time to fill up a level is getting shorter if $\lambda$ is asymptotically larger than $\mu$. For example, for $\mu = n$ and $\lambda = n^{1.5}$, we can now derive an optimization time of $O(n\lambda) = O(n^{2.5})$, while the previous result was $O(n \lambda \log(\lambda)) = O(n^{2.5}\log(n))$.

\section{Conclusion}

In this work, we prove three drift results for multiplicatively increasing drift. Since the desired hitting time bound of order ${\log(n)/\min\{\delta,\log(1+\delta)\}}$, which implies that the process behaves similarly to the deterministic process, can only be obtained under additional assumptions, we formulate our results for processes in which each state $X_{t+1}$ is distributed according to a binomial distribution with expectation $(1+\delta) X_{t}$ (or better, in the domination sense). 

As main application for our drift results, we prove a stronger version of the level-based theorem. It in particular has the asymptotically right dependence on $1/\delta$, which is near-linear. Previous level-based theorems only show a dependence roughly of order $\delta^{-5}$~\cite{DangL16} or $\delta^{-2}$~\cite{CorusDEL18plus}. This difference can be significant in applications with small $\delta$, e.g., the result on fitness-proportionate selection~\cite{DangL16}, which has $\delta = \Theta(1/n)$.

An equally interesting progress from our new level-based theorem is that its relatively elementary proof gives more insight in the actual development of such processes. It thus tells us in a more informative manner how certain population-based algorithms optimize certain problems. Such additional information can be useful to detect bottlenecks and improve algorithms. Also, the individual building blocks of our drift analysis may find separate applications.

%
In terms of future work, we note that there are processes showing multiplicative up-drift where the next state is not described by a binomial distribution. One example are population-based algorithms using plus-selection, where, roughly speaking, $X_{t+1} \sim X_{t} + \Bin(\lambda,X_{t}/\lambda)$. We are optimistic that such processes can be handled with our methods as well. We did not do this in this first work on multiplicative up-drift since such processes can also be analyzed with elementary methods, e.g., exploiting that the process is non-decreasing and with constant probability attains the expected progress. Nevertheless, extending our drift theorems to such processes should give better constants and a more elegant analysis, so we feel that this is also an interesting goal for future work.

}

\end{document}